\documentclass[lettersize,journal]{IEEEtran}
\usepackage{amsmath,amsfonts}
\usepackage{algorithmic}
\usepackage{algorithm}
\usepackage{array}
\usepackage[caption=false,font=normalsize,labelfont=sf,textfont=sf]{subfig}
\usepackage{textcomp}
\usepackage{stfloats}
\usepackage{url}
\usepackage{verbatim}
\usepackage{graphicx}
\usepackage{cite}

\usepackage{url}
\usepackage{graphicx}
\usepackage{natbib}
\setcitestyle{numbers,square,comma}
\usepackage{multicol}
\usepackage{amssymb}

\usepackage{tabu,multirow}
\usepackage[leftcaption]{sidecap}

\usepackage{color}
\usepackage{enumitem}
\usepackage{verbatim}
\renewcommand{\arraystretch}{1.3} 
\usepackage{float}  
\usepackage{textcomp}
\usepackage{multicol}
\usepackage{booktabs}
  
\usepackage{eso-pic}
\usepackage{xspace}
  
\newtheorem{theorem}{Theorem}

\newtheorem{proof}{Proof}

\newtheorem{proposition}{Proposition}
\usepackage[pagebackref=false,breaklinks=true,colorlinks,bookmarks=false,urlcolor=magenta]{hyperref}

\begin{document}

\title{Manifold Constraint Regularization for Remote Sensing Image Generation}

\author{Xingzhe~Su\IEEEauthorrefmark{1},
        Changwen~Zheng\IEEEauthorrefmark{1},
           Wenwen~Qiang,
           Fengge~Wu,
           Junsuo~Zhao,    Fuchun~Sun,~\IEEEmembership{Fellow,~IEEE},
Hui~Xiong,~\IEEEmembership{Fellow,~IEEE}
           

\IEEEcompsocitemizethanks{
\IEEEcompsocthanksitem \IEEEauthorrefmark{1} {These authors contributed equally.}
\IEEEcompsocthanksitem Xingzhe Su, Changwen Zheng, Wenwen Qiang, Fengge Wu and Junsuo Zhao are with the National Key Laboratory of Space Integrated Information System, Institute of Software Chinese Academy of Sciences, University of Chinese Academy of Sciences, Beijing, China. E-mail: \{xingzhe2018, changwen, qiangwenwen, fengge, junsuo\}@iscas.ac.cn.
\IEEEcompsocthanksitem Fuchun Sun is with the National Key Laboratory of Space Integrated Information System, Department of Computer Science and Technology, Tsinghua University, Beijing, China. E-mail: fcsun@tsinghua.edu.cn.
\IEEEcompsocthanksitem Hui Xiong is with the Hong Kong University of Science and Technology, China. E-mail: xionghui@ust.hk.
\IEEEcompsocthanksitem Corresponding author: Wenwen Qiang (qiangwenwen@iscas.ac.cn).
  }
}

\markboth{Journal of \LaTeX\ Class Files,~Vol.~14, No.~8, March~2024}%
{Shell \MakeLowercase{\textit{et al.}}: A Sample Article Using IEEEtran.cls for IEEE Journals}

\IEEEpubid{0000--0000/00\$00.00~\copyright~2024 IEEE}

\maketitle

\begin{abstract}
Generative Adversarial Networks (GANs) have shown notable accomplishments in remote sensing domain. However, this paper reveals that their performance on remote sensing images falls short when compared to their impressive results with natural images. This study identifies a previously overlooked issue: GANs exhibit a heightened susceptibility to overfitting on remote sensing images. 
To address this challenge, this paper analyzes the characteristics of remote sensing images and proposes manifold constraint regularization, a novel approach that tackles overfitting of GANs on remote sensing images for the first time. Our method includes a new measure for evaluating the structure of the data manifold. Leveraging this measure, we propose the manifold constraint regularization term, which not only alleviates the overfitting problem, but also promotes alignment between the generated and real data manifolds, leading to enhanced quality in the generated images. The effectiveness and versatility of this method have been corroborated through extensive validation on various RS datasets and GAN models. The proposed method not only enhances the quality of the generated images, reflected in a 3.13\% improvement in Frechet Inception Distance (FID) score, but also boosts the performance of the GANs on downstream tasks, evidenced by a 3.76\% increase in classification accuracy. The source code is available at \href{https://github.com/rootSue/Manifold-RSGAN}{https://github.com/rootSue/Manifold-RSGAN}.
\end{abstract}


\section{Introduction}
\label{intro}

\IEEEPARstart{I}{n} recent years, the field of artificial intelligence has witnessed the emergence of Generative Adversarial Networks (GANs) \cite{10.5555/2969033.2969125}, a paradigm-shifting technology that has significantly advanced the capabilities of image generation. 
Among the myriad domains benefiting from this technology, remote sensing (RS) imagery stands out as a particularly promising area. Here, GANs have demonstrated their potential in data generation or augmentation \cite{MartaGAN,AttentionGAN,MFGAN, chen2021adversarial}, haze or cloud removal \cite{hu2020unsupervised, zhao2023remote,li2020thin, ma2023cloud}, and image super-resolution \cite{jiang2019edge,xiong2020improved,guo2023improved,song2023dbsagan}. 
In this paper, our primary focus is on the application of GANs to RGB images within the RS domain.

\begin{figure}[tpb]
\centering
\subfloat[]{\includegraphics[width=.5\columnwidth]{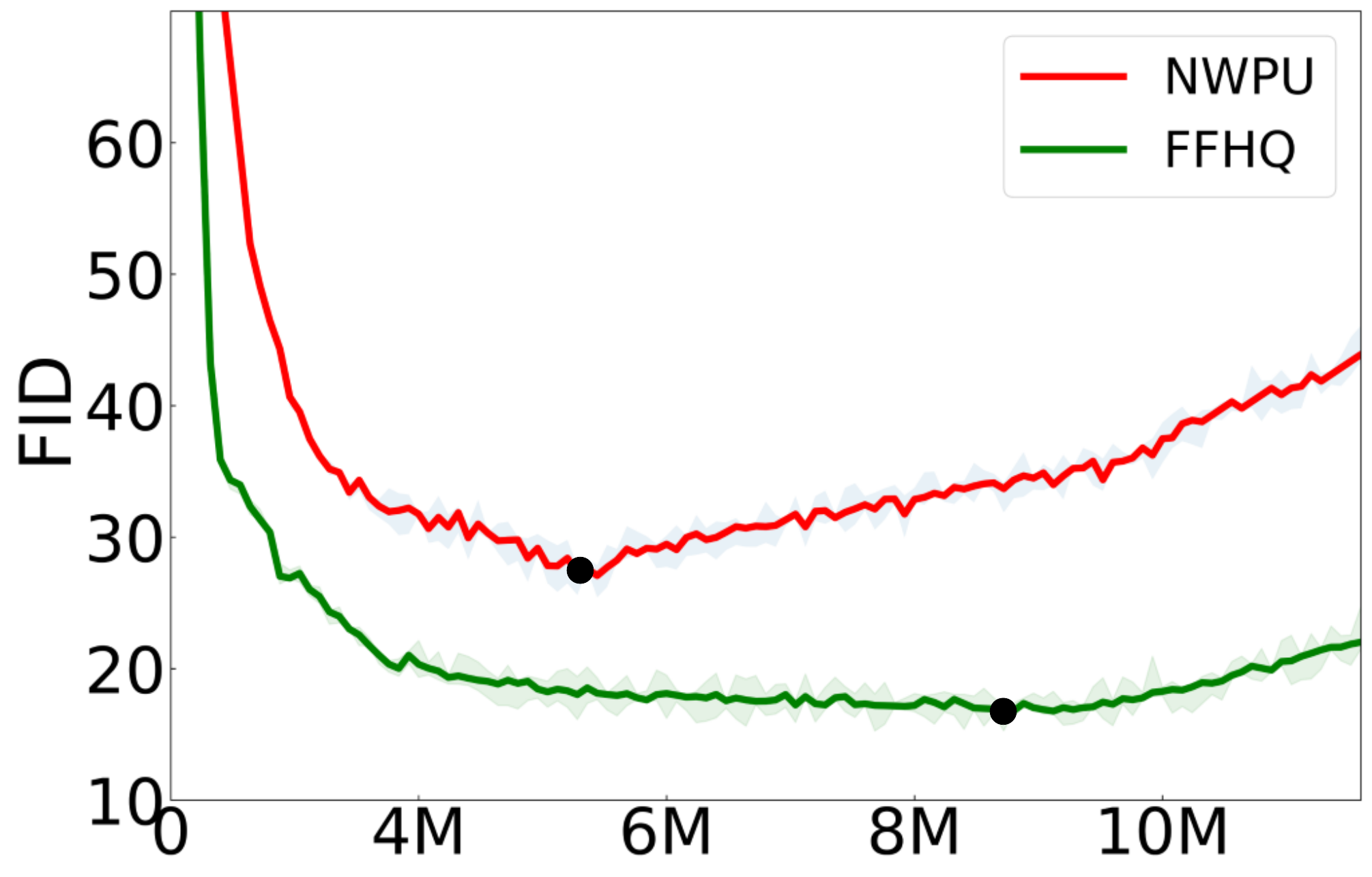}}
\subfloat[]{\includegraphics[width=.5\columnwidth]{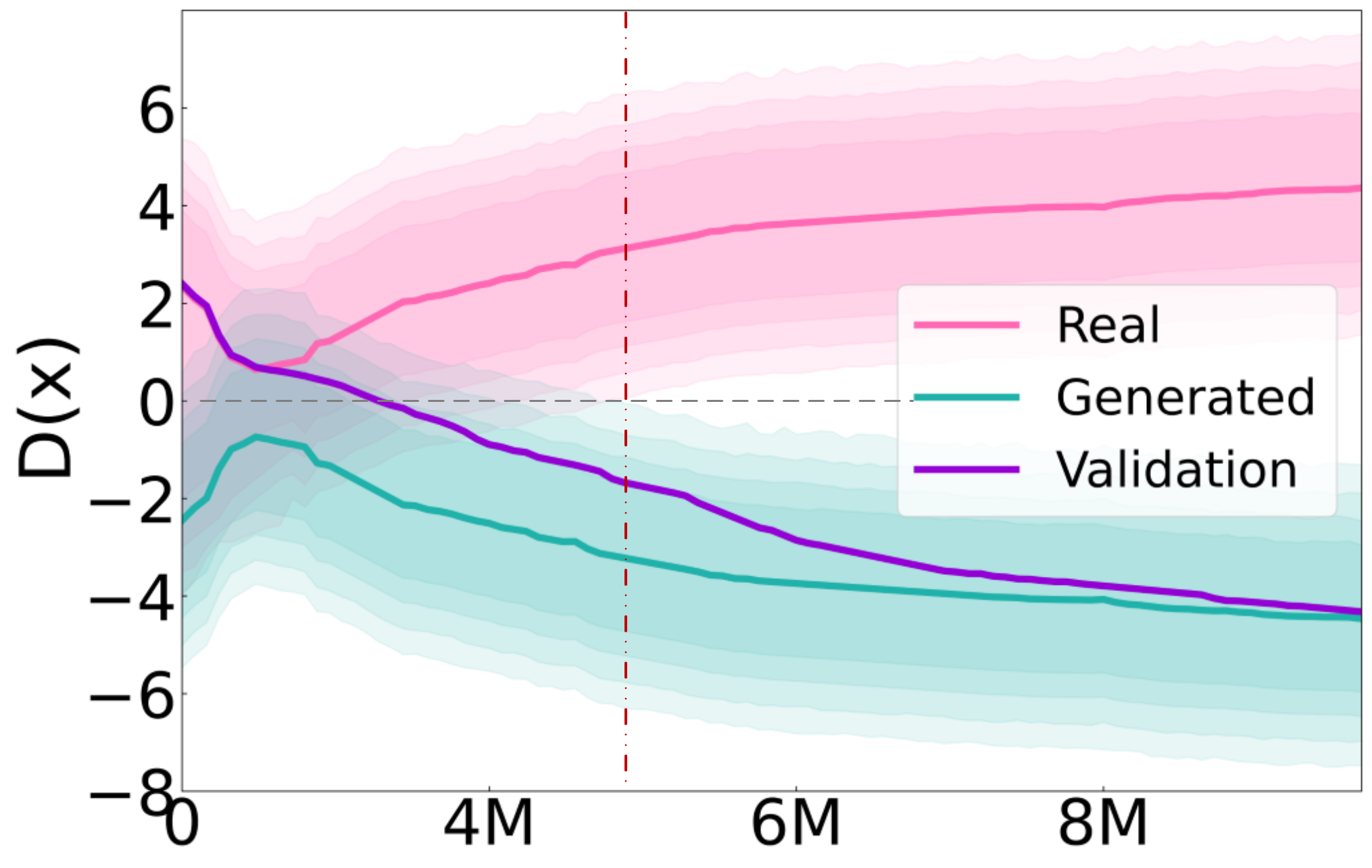}}
\caption{The horizontal axis indicates the training process (the number of real images shown to the discriminator). (a) Training curves of StyleGAN2 on NWPU and FFHQ datasets. We randomly sample 30,000 training images from these two datasets, respectively. The GAN model diverges earlier when trained on NWPU dataset. (b) The outputs of the discriminator during training on NWPU dataset. As training progresses, the validation set of real images is misclassified as generated images, highlighting the discriminator’s overfitting issue on NWPU dataset.}
\label{fig0-0}
\end{figure}


Current GAN methods achieve remarkable success with natural images. However, our research reveals a performance gap when applied to RS images, as shown in Fig.\ref{fig0-0}(a). The vertical axis is the Fréchet Inception Distance (FID) score \cite{FID}, a commonly-used image quality evaluation metric, with lower FID scores indicating higher image quality. We train the same GAN model on datasets of identical size: a natural image dataset (FFHQ \cite{StyleGAN}) and an RS image dataset (NWPU \cite{NWPU}). Interestingly, the model diverges earlier when trained on RS data compared to natural data (Fig. \ref{fig0-0}(a)), leading to lower quality generated RS images. To understand the reasons behind, we analyze the discriminator's outputs for real and generated images during training on NWPU dataset. Initially, the distributions of outputs for real and fake images overlap (Fig.\ref{fig0-0}(b)). However, as the discriminator becomes more confident, these distributions gradually drift apart. Additionally, the accuracy of the discriminator on a separate validation set decreases as training progresses.
Based on these observations, we conclude that GANs are more susceptible to overfitting on RS image dataset compared to natural image dataset. This overfitting leads to earlier divergence during training, ultimately resulting in lower quality generated RS images. More experiments are avaliable in Section \ref{pa}. As far as we know, this is the first study that identify the overfitting issue of GANs on RS data.

\IEEEpubidadjcol


To understand why GANs are more prone to overfitting on RS data, we analyze the key differences between RS and natural image datasets. Compared to natural images, RS images typically cover larger area, encompassing a wider variety of scenes and richer content. Consequently, we hypothesize that the intrinsic dimension of the RS dataset is larger than that of natural dataset. Highlighting a concept from \cite{narayanan2010sample}, learning a manifold requires a number of samples that grows exponentially with the manifold’s intrinsic dimension. Therefore, the discriminator might need more training images from the RS dataset, compared to natural images, to effectively capture the underlying data manifold. Our findings in Section \ref{pa} support this hypothesis, confirming that the intrinsic dimension of RS data is indeed higher than that of natural images. Consequently, with datasets of the same size, GANs are more likely to overfit on RS data compared to natural data.

The overfitting problem can significantly hinder the performance of GANs. When the discriminator becomes overfit to the training samples, its feedback to the generator becomes less meaningful, leading to training divergence, excessive memorization, and limited generalization \cite{stylegan2-ada,lecam}. Consequently, the performance of GANs on tasks like data augmentation could also be compromised. While current research on GANs for RS data focuses on modifying network architectures and loss functions to improve performance on downstream tasks \cite{dash2023review,jozdani2022review}, the issue of overfitting in the discriminator specifically for RS data is overlooked. This gap in research presents a unique opportunity to improve the overall effectiveness of GANs for RS applications. 

Motivated by our findings, we propose to leverage the real data manifold to mitigate the overfitting problem of GANs on RS data. Specifically, we propose the manifold constraint regularization method (MCR) and integrate the regularization term into GANs’ loss functions. MCR presents the discriminator with a more challenging task: capturing the underlying manifold of the real data, rather than simply memorizing the limited variations within the training dataset. This approach helps to eliminate complex model that performs well solely on training data and promotes model that performs well across the entire data manifold, thereby mitigating overfitting. Additionally, by minimizing MCR term, the generator aims to align its output manifold with that of the real images, capturing the authentic characteristics of the data and enhancing generative performance. Our methods is computationally efficient and integrates seamlessly within existing GAN architectures, requiring no major network modifications.

In summary, the contributions of this paper are:

\begin{itemize}
\item We identify a previously overlooked issue that the discriminator in GANs is prone to overfitting on RS image datasets, compared to natural images.

\item We empirically demonstrate that RS datasets have higher intrinsic dimension than natural datasets, which leads to the discriminator easily overfitting to the RS datasets.

\item To address overfitting, we propose MCR method, which consists of two components: feature distribution compactness measure and data manifold evaluation function. 

\item We theoretically prove that MCR can evaluate the distance between different manifolds. Extensive experiments across multiple RS datasets and GAN models verifies the effectiveness of our method. 
\end{itemize}

In the remainder of this paper, we commence with a review of related works in Section \ref{rw}, and analyze plausible reasons for the subpar results of GAN models on RS images in Section \ref{sec:3}. Then we provide details of our proposed method in Section \ref{method}, followed by theoretical analysis of our approach in Section \ref{sec:5}. We show the substantially-improved performances of our method over standard GAN models and related techniques for solving overfitting in Section \ref{sec:exp}. Finally, we wrap up with a discussion in Section \ref{sec:con}.

\section{Related Work}
\label{rw}
\subsection{Generative adversarial networks}

GANs are notorious for training instability and mode collapse. Various adversarial losses have been proposed to stabilize the training or improve the convergence of the GAN models \cite{10.5555/2969033.2969125, mroueh2017fisher,WGAN}. Additionally, numerous efforts have been made to address this issue using regularization methods \cite{R1,SNGAN,WGAN-GP,veegan}, or modifying network architectures \cite{biggan,StyleGAN,StyleGAN2,esser2021taming,karras2021alias,sauer2022stylegan,sauer2023stylegan}. Other than these problems, the overfitting of the discriminator is also a common challenge.

The overfitting problem occurs when the discriminator becomes overly complex with a large number of parameters, resulting in memorization of the training data rather than learning the underlying data distribution. 
To mitigate this issue, several strategies have been proposed, which can be divided into the following categories. The first category is data augmentation methods \cite{zhang2019consistency,APA,zhao2020differentiable,stylegan2-ada}, which utilizes traditional data augmentation methods, such as rotation and color transformation, to increase the amount of training data. The second type is regularization, which adds regularization term to the loss function of the discriminator \cite{lecam,yang2021data}, allowing it to learn more discriminative representations under limited training data. InsGen \cite{yang2021data} proposes a contrastive learning objective to enhance the adversarial loss in the few-shot generation setting. AdaptiveMix \cite{liu2023adaptivemix} narrows down the distance between hard samples and easy samples, where hard samples are regarded as the samples that are difficult for discriminator to classify. The third category is model architecture improvement. FastGAN \cite{liu2020towards} introduces a self-supervised discriminator and a Skip-Layer channel-wise Excitation module for efficient few-shot image synthesis. MoCA \cite{li2022prototype} proposes prototype-based memory modulation module to improve the generator network of a GAN. The proposed method in this paper falls into the second category by introducing manifold constraint regularization, a novel approach that addresses the overfitting problem of GANs on RS data from the manifold perspective for the first time.

Previous works \cite{park2017mmgan, li2018mrgan} introduce geometry constraints to GANs loss functions. These methods utilize statistical mean and radius to approximate the geometry of the real data, but they lack accurate constraints on the data manifold and may lose important geometrical information. Other approaches \cite{ni2022manifold,bang2021mggan,khayatkhoei2018disconnected} offer more precise control, but they require modifications to the network architecture. In contrast, our approach avoids the limitations of prior work by being applicable to any GAN model without requiring network architecture changes, offering both efficiency and effectiveness. 

\subsection{GAN in the RS field}

In this paper, our primary focus is on RGB images in the RS domain. Existing GAN models in the RS field can be categorized into two main types based on their applications.

The first type revolves around data generation or augmentation \cite{MartaGAN,AttentionGAN,MFGAN, chen2021adversarial,wu2023fully}. For instance, Lin \textit{et al.} introduced MARTAGAN \cite{MartaGAN}, marking the pioneering application of GANs to remote sensing images. MARTAGAN, an extension of DCGAN \cite{DCGAN}, introduce enhancements like a multi-feature layer in the discriminator and feature loss in the generator to improve image representations. This work demonstrated the potential of GANs to augment datasets and enhance unsupervised classification accuracy. Similarly, Yu \textit{et al.} propose Attention GANs \cite{AttentionGAN}, designed for unsupervised classification tasks, by integrating attention mechanisms into GANs to bolster the discriminator's representation capabilities. Furthermore, Wu \textit{et al.} \cite{wu2023fully} propose a comprehensive change detection framework utilizing GANs to tackle various RS change detection tasks. Their approach incorporates an image-to-image generator to capture spectral and spatial variations between multi-temporal images.

The second type pertains to GANs applied in image enhancement tasks, such as image super-resolution \cite{jiang2019edge,xiong2020improved,guo2023improved,song2023dbsagan} and cloud or haze removal \cite{hu2020unsupervised, zhao2023remote,li2020thin, ma2023cloud}. Examples include the work by Wu \textit{et al.} \cite{jiang2019edge}, which introduce a GAN-based edge-enhancement network designed to reconstruct high-resolution RS images, employing adversarial learning to restore edge details obscured by noise. Guo \textit{et al.} \cite{guo2023improved} have refined the Super Resolution Generative Adversarial Network \cite{ledig2017photo} by altering both the internal and external connections of the residual block and modifying the loss function to incorporate the Charbonnier penalty for improved performance. In the realm of haze removal, Hu \textit{et al.} \cite{hu2020unsupervised} propose edge-sharpening cycle-consistent adversarial network (ES-CCGAN), ubstituting the standard residual network with a dense convolutional network to enhance edge clarity in images. The edge-sharpening loss function of ES-CCGAN is designed to further recover clear ground-object edges. For cloud removal, Li \textit{et al.} \cite{li2020thin} develope a semi-supervised technique, CR-GAN-PM, merging GANs with a physical model to address cloud distortions in unpaired images from various regions.

Our contribution aligns with the first category, focusing on addressing the overfitting challenge in GANs on RS image generation. Our goal is to enhance the quality of generated images and improve the discriminator’s effectiveness, thereby elevating GAN performance in data generation and augmentation tasks. This paper marks a pioneering effort to specifically tackle issues related to the overfitting in the context of RS image generation, underscoring our novel approach to improving GAN applications within this domain.

\section{Preliminary Study and Analysis}
\label{sec:3}
In this section, we first introduce the basic framework of GANs. Then we explore the overfitting problem of GAN models on RS image generation tasks. Finally, we analyze plausible reasons about the poor results of the GAN models on RS images and derive the motivation for this paper.

\subsection{Preliminary GANs}

The GAN model aims to learn the distribution of training samples. Based on the idea of the zero-sum game, a GAN model consists of a generator $G$ and a discriminator $D$. The generator aims to generate realistic samples to fool the discriminator, while the discriminator tries to distinguish between real and fake samples. When the model reaches the final equilibrium point, the generator will model the target distribution and produce counterfeit samples, which the discriminator will fail to discern. Let $\mathcal{L}_{\mathrm{D}}$ and $\mathcal{L}_{\mathrm{G}}$ denote the loss functions of the discriminator $D$ and the generator $G$, respectively. The training of the GAN frameworks can be generally illustrated as follows:
\begin{eqnarray}
\min _{\mathrm{D}} \mathcal{L}_{\mathrm{D}} & = & -\underset{x \sim {P_{\mathrm{data}}}}{\mathbb{E}}\left[D(x)\right]+\underset{z \sim p_{\mathrm{z}}}{\mathbb{E}}\left[D(G(z))\right]
\end{eqnarray}
\begin{eqnarray}
\min _{\mathrm{G}} \mathcal{L}_{\mathrm{G}} & = & -\underset{z \sim p_{\mathrm{z}}}{\mathbb{E}}\left[D(G(z))\right]
\end{eqnarray}
where $P_{\mathrm{data}}$ denotes the real data distribution, and $P_{\mathrm{z}}$ is usually the \textit{normal distribution}.


\subsection{Problem Analysis}
\label{pa}

\begin{figure}[hb]
\centering
\subfloat[]{\includegraphics[width=.7\columnwidth]{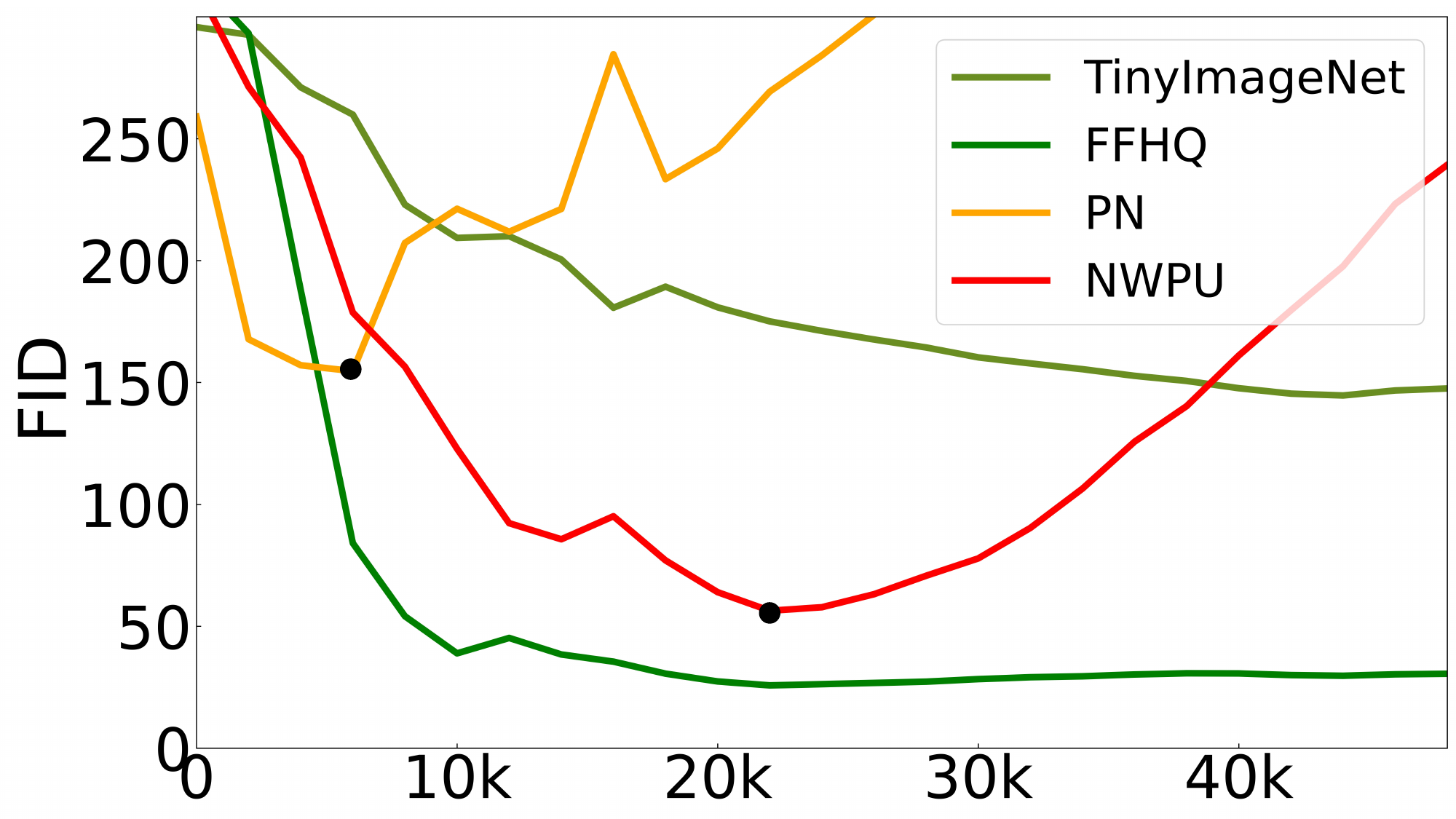}} \vspace{0.5pt}
\subfloat[]{\includegraphics[width=.7\columnwidth]{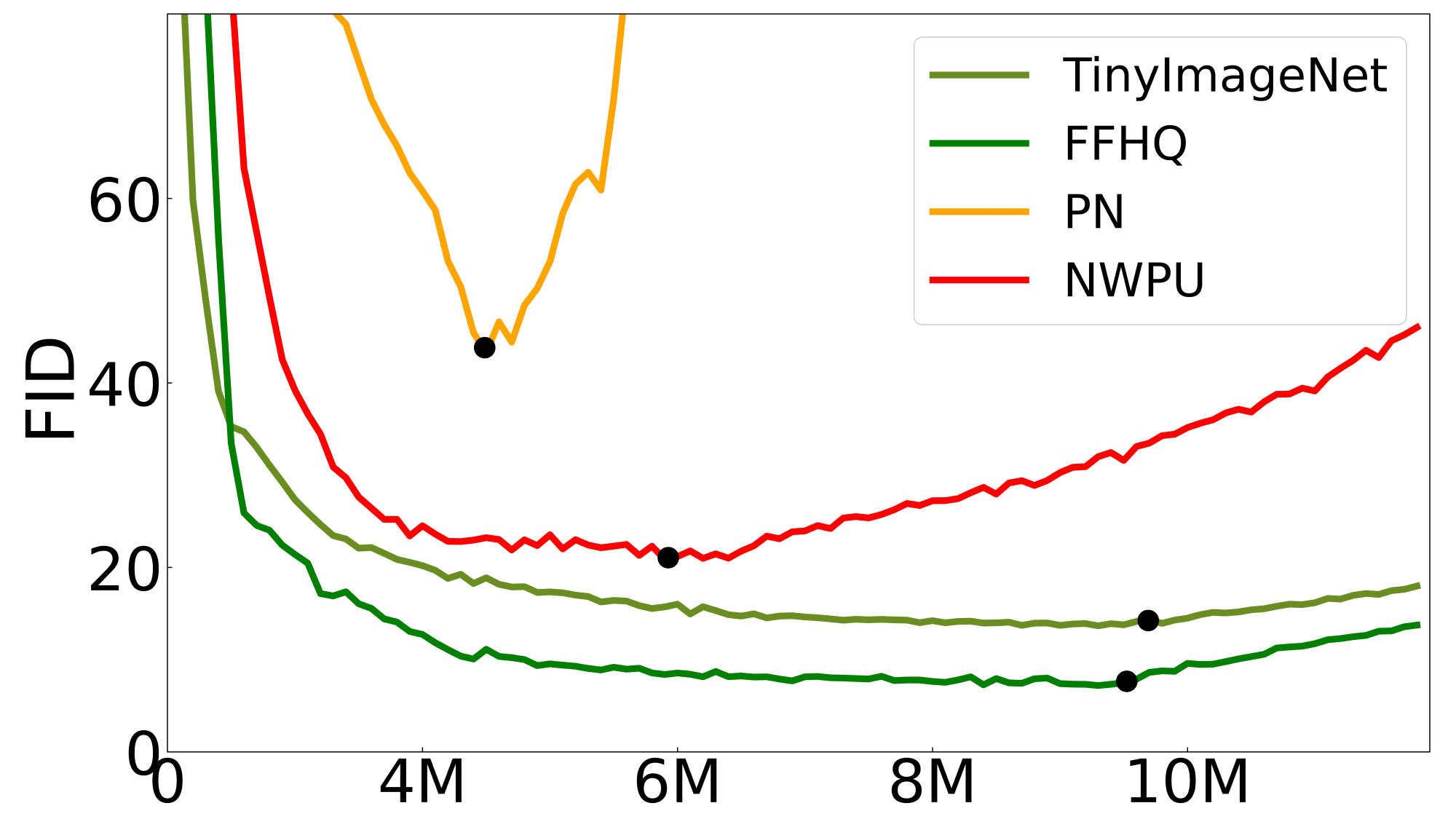}}
\caption{(a) Training curves of BigGAN on TinyImageNet, FFHQ, NWPU and PN datasets. The horizontal axis is the number of training steps. (b) Training curves of StyleGAN2 on TinyImageNet, FFHQ, NWPU and PN datasets. The horizontal axis indicates the training process (the number of real images shown to the discriminator). These training curves highlight an earlier divergence on RS datasets compared to natural ones, emphasizing our claim about the extent of overfitting in RS data.}
\label{fig-model}
\end{figure}

As previously stated, GANs are more susceptible to overfitting on RS image dataset compared to natural image dataset. Notably, the GAN model exhibits earlier divergence when trained on the RS images (NWPU) compared to the natural images (FFHQ), as depicted in Fig. \ref{fig0-0}. To extend our investigation on the prevalence of overfitting across different datasets and GAN architectures, we conduct further experiments using the most popular GAN models: StyleGAN2 \cite{StyleGAN2} and BigGAN \cite{biggan}. We employ NWPU and PatternNet (PN) \cite{patternnet} for RS image datasets, and FFHQ256 and TinyImageNet \cite{tiny-imagenet} for natural image datasets. Each dataset comprises a random sample of 30,000 images. The experiment results are shown in Fig.\ref{fig-model}. The training curves of these GAN models underscore an earlier divergence on RS datasets compared to their natural counterparts, reinforcing our inference regarding the severity of overfitting in RS data. 



To identify the potential causes underlying this phenomenon, we analyze the inherent differences between RS RGB image and natural image. Compared to their natural counterparts, RS images have a higher spatial resolution and a larger coverage area, encompass a wider variety of scenery, and exhibit richer content. Deep learning has an underlying manifold assumption on the training data, i.e., high-dimensional data can be embedded into low-dimensional manifolds. This inherent assumption allows deep learning models to effectively handle high-dimensional data and achieve remarkable performance, as explicitly confirmed by Pope et al. \cite{pope2020intrinsic}. It is well-established that learning a manifold requires a number of samples that grows exponentially with the manifold’s intrinsic dimension \cite{narayanan2010sample}. Hence, we conjecture that GANs are more easily overfitting on RS image generation tasks because of the higher intrinsic dimension of RS images. Under limited training samples, the discriminator struggles to learn the intricate manifold of the RS images, resulting in a tendency to memorize the training data and overfit to the dataset's limited variations.
\begin{figure*}[htpb]
\centering
\includegraphics[width=0.65\textwidth]{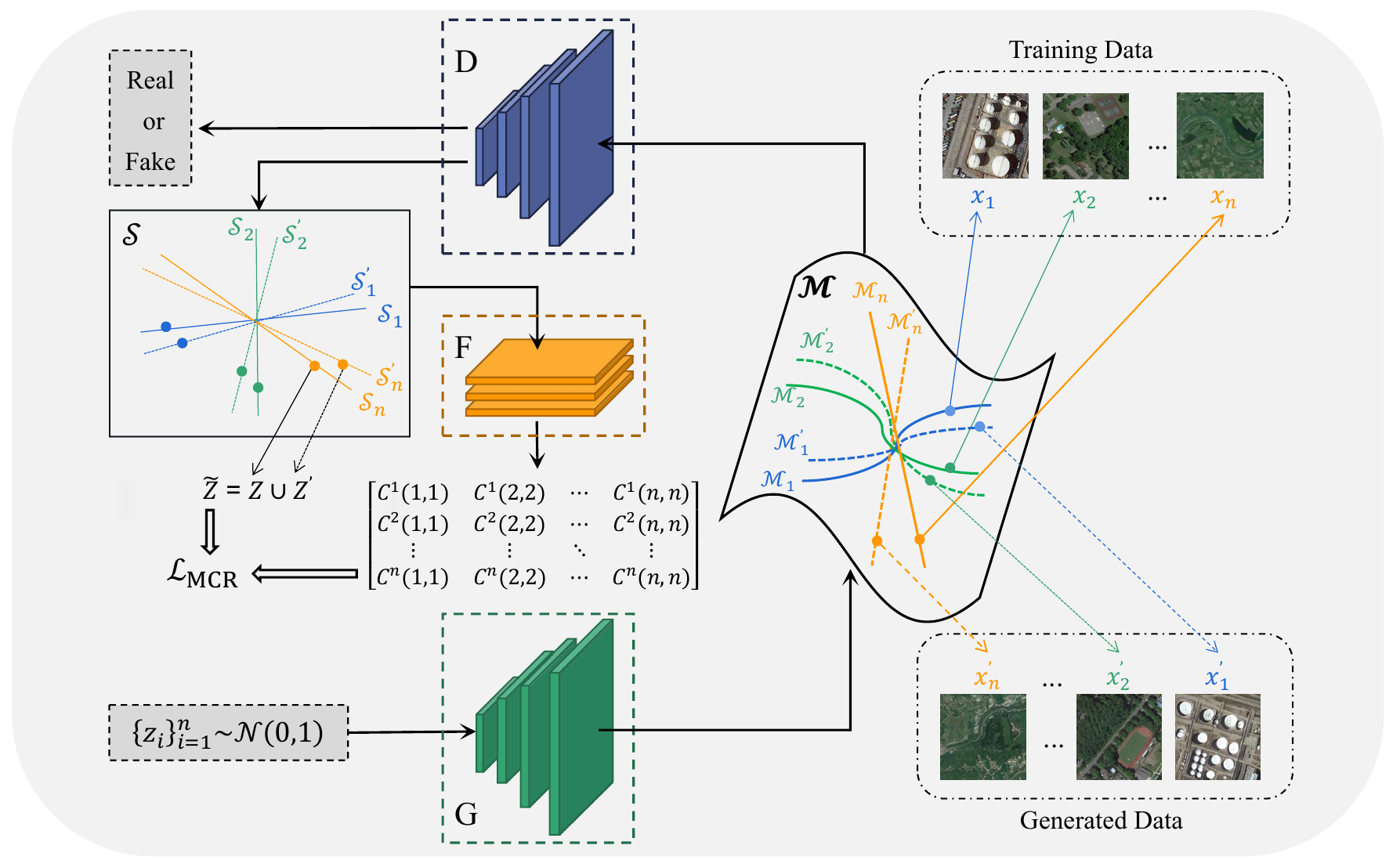}
\caption{The overall architecture of our method. $G$ and $D$ are generator and discriminator. $F$ is the MLP network for relationship matrix $C$. $\mathcal{M}$ represents manifold, and $\mathcal{S}$ represents subspace.}
\label{fig-2}
\end{figure*}

To validate our hypothesis, we employ the Maximum Likelihood Estimation (MLE) method \cite{MLE},
the same approach used by Pope et al. \cite{pope2020intrinsic}, to estimate the intrinsic dimension of the RS images.
\begin{eqnarray}
{m}_{k} = \left[\frac{1}{n(k-1)} \sum_{i=1}^{n} \sum_{j=1}^{k-1} \log \frac{T_{k}\left(x_{i}\right)}{T_{j}\left(x_{i}\right)}\right]^{-1},
\end{eqnarray}
where $T_{j}(x)$ is the Euclidean $\left(\ell_{2}\right)$ distance from $x$ to its $j^{t h}$ nearest neighbor, and $n$ is the number of samples. We conduct this experiment with different values for $k=3, 5, 10, 20$ and a fixed sample size $n = 30,000$. The results for the NWPU, PN, FFHQ, and TinyImageNet datasets are presented in Table \ref{table0}. As expected, the intrinsic dimensions of the NWPU and PN datasets are indeed higher compared to those of the FFHQ and TinyImageNet datasets.  

\begin{table}[]
\centering
\caption{The Intrinsic Dimension Estimated by MLE.}
\renewcommand\arraystretch{1.2}
\begin{tabular}{ccccl}
\toprule
\textbf{Dataset} & \textbf{FFHQ} & \textbf{TinyImageNet} & \textbf{NWPU} & \textbf{PN} \\ \midrule
MLE (k=3)         & 35            & 38                    & 52            & 42          \\
MLE (k=5)         & 34            & 36                    & 53            & 44          \\
MLE (k=10)        & 33            & 33                    & 48            & 44          \\
MLE (k=20)        & 31            & 30                    & 44            & 42          \\ \bottomrule
\end{tabular}
\label{table0}
\end{table}


Based on our experiments, we can infer that discriminator is more susceptible to overfitting on RS data due to the higher intrinsic dimension of the RS dataset compared to the natural dataset. To address this issue, we propose introducing the data manifold to constrain the discriminator. By encouraging the learned features to align with the structure of the real data manifold, we can guide the discriminator to learn the underlying structure of the data distribution and avoid overfitting to the local features of the training data. This approach favors model that performs well across the entire data manifold, promoting the discriminator to generalize well beyond the training data, thereby mitigating the overfitting problem.

\section{Methods}
\label{method}

In this section, we initially present an assumption regarding the real data manifold, accompanied by three critical attributes that ideal representations of images ought to exhibit. Subsequently, we introduce a novel measure to evaluate the data manifold of feature distribution. Building upon this, we propose novel regularization term, which we call manifold constraint regularization (MCR). The framework of the proposed method is shown in Fig.\ref{fig-2}. 

\subsection{Preliminary Assumptions on Data Manifold}

In practice, it is intractable to learn the data manifold in the high-dimensional ambient space $\mathbb{R}^{D}$ \cite{pope2020intrinsic,brown2022verifying}.
Therefore, following the approach in \cite{brown2022verifying,liu2021rectifying}, we assume that the real data manifold comprises of a union of low-dimensional nonlinear submanifolds $\cup_{j = 1}^{k} \mathcal{M}_{j} \subset \mathbb{R}^{D}$, where each submanifold $\mathcal{M}_j$ is of dimension $d_j<<D$. Each submanifold $\mathcal{M}_{j} \subset \mathbb{R}^{D}$ can be transformed to a linear subspace $\mathcal{S}_{j} \subset \mathbb{R}^{d}$, which we refer to the feature space. 
From this assumption, we can infer that images residing on the same submanifold share similarities, whereas images on different submanifolds exhibit distinct characteristics. Consequently, the ideal features of these images should exhibit the following attributes:
\begin{itemize}
\item Between-submanifold Discrepancy: Features of images from different submanifold should be highly uncorrelated.
\item Within-submanifold Similarity: Features of images from the same submanifold should be relatively correlated.
\item Maximally Variance: Features should have as large dimension as possible to cover all the submanifolds and be variant for the same submanifold. 
\end{itemize}
Therefore, in order to accurately capture the real data manifold, it’s desirable for the features from different submanifolds to be as uncorrelated as possible. On the other hand, features from the same submanifold should display a high correlation and coherence. Simultaneously, the features should exhibit maximum variance to cover all possible submanifolds. By learning features that adhere to these properties, the discriminator can be considered to have successfully captured the underlying structure of the real data manifold. 

\subsection{Data Manifold Evaluation}

Based on the data manifold assumption, features from different submanifolds should have a sparse feature distribution. Conversely, features within the same submanifold should have a compact feature distribution. Our first step is to establish a metric for the compactness of the feature distribution by examining the relationships between features. Secondly, we design data manifold evaluation function based on this measure.

\textbf{Step I}. We propose leveraging singular values to quantify compactness. Recall that singular values (represented by $\lambda_{i}$ here) capture the amount of variance explained by each principal component in a data matrix. Singular vectors corresponding to larger singular values represent the principal stretching directions of the data. In simpler terms, a larger number of significant singular values indicates a more uniform distribution, while fewer significant values suggest a more compact distribution. Based on this principle, we propose the following measure:

\begin{eqnarray}
\mathcal{V}(Z) = \sum_{i = 1} \lambda_{i}^{2}=\operatorname{Tr}(Z Z^{\mathrm{T}}).
\label{eq3}
\end{eqnarray}

A larger value of $\mathcal{V}(Z)$ indicates a broader span of the singular vectors and thus greater uniformity of the data. Considering computational efficiency, we opt for the trace of $Z Z^{T}$ in our method. In the Section \ref{sec:5}, we theoretically prove $\mathcal{V}(Z)$ can measure the compactness of a distribution. 

\textbf{Step II}. Building upon Eq.~\ref{eq3}, we introduce a novel data manifold evaluation function $\mathcal{L}_{\operatorname{Tr}}(Z)$. This function leverages the concept of feature compactness to assess how well the learned features capture the underlying data manifold.

According to the aforementioned assumption, features of different submanifolds should be maximally uncorrelated with each other. Therefore, they together should span a space of the largest possible volume or dimension, and $\mathcal{V}(Z)$ should be as large as possible. Conversely, learned features of the same submanifold should be highly correlated and coherent. Therefore, they should only span a space of a very small volume or dimension. To capture these contrasting properties, our evaluation function, $\mathcal{L}_{\operatorname{Tr}}(Z)$, incorporates both the overall and within-submanifold compactness:
\begin{eqnarray}
\label{eq4}
\mathcal{L}_{\operatorname{Tr}}(Z) & = & \frac{1}{2 n}(\operatorname{Tr}(Z Z^{\mathrm{T}})-\sum_{j=1}^{k} \operatorname{Tr}(Z C^j Z^{\mathrm{T}}))
\end{eqnarray}
where $Z=\left[z_{1}, \ldots, z_{n}\right] \subset \mathbb{R}^{d \times n}$ denotes the representations of images, $C=\left\{C^{j} \in \mathbb{R}^{n \times n}\right\}_{j=1}^{k}$ is a set of positive diagonal matrices whose diagonal entries denote the membership of the $n$ samples in the $k$ submanifolds. If the sample $x_i$ belongs to the submanifold $j$, then $C^{j}(i,i) = 1$. Otherwise, $C^{j}(i,i) = 0$, where $i \in \left\{ {1,...,n} \right\}, j \in \left\{ {1,...,k} \right\}$. 

The first term, $\operatorname{Tr}(Z Z^{\mathrm{T}})$, measures the overall compactness of the entire feature set, as defined by Eq.~\ref{eq3}. The second term, $\sum_{j=1}^{k} \operatorname{Tr}(Z C^j Z^{\mathrm{T}})$, represents the compactness within each individual class. The difference between these two terms reflects the ``dispersibility'' between features from different submanifolds, and a larger difference signifies better separation. Therefore, a higher value of $\mathcal{L}_{\operatorname{Tr}}(Z)$ indicates that the learned features effectively capture the intrinsic structure of the data manifold, with features from different submanifolds being well-separated and features within a submanifold being well-clustered.

\subsection{Manifold Constraint Regularization}

Building on the concept of data manifold evaluation introduced earlier, we propose a novel regularization term, $\mathcal{L}_{\mathrm{MCR}}(Z, Z^{\prime})$, to assess how well the generated image manifold aligns with the real image manifold. 
\begin{equation}
\label{eq5}
\begin{split}
\mathcal{L}_{\mathrm{MCR}}(Z, Z^{\prime}) =& \frac{1}{2 n} (\operatorname{Tr}(\widetilde{Z} \widetilde{Z}^{\mathrm{T}})- \frac{1}{2} \sum_{j=1}^{k} \operatorname{Tr}(Z C^j Z^{\mathrm{T}}) \\ 
& -\frac{1}{2} \sum_{j=1}^{k} \operatorname{Tr}(Z^{\prime} {C^{\prime}}^j {Z^{\prime}}^{\mathrm{T}})),
\end{split}
\end{equation}
where $\widetilde{Z} = Z \cup Z^{\prime}$, $Z$ and $Z^{\prime}$ represent the feature representations of real and generated images, respectively. $C^{j}$ and ${C^{\prime}}^j$ are diagonal matrices encoding submanifold membership information for real and generated data, respectively. $k$ is the number of submanifolds within the data. 

The first term of $\mathcal{L}_{\mathrm{MCR}}$ measures the compactness of the joint distribution of real and generated images, i.e. the volume of the space spanned by the real and generated features jointly. The second and third terms measure the compactness of the real and generated distributions individually. If the generated image manifold aligns well with the real image manifold, then the combined distribution should be compact, which translates to a lower value for $\mathcal{L}_{\mathrm{MCR}}$. Conversely, a larger value indicates a significant deviation between the manifolds of real and generated data. We theoretically prove that the $\mathcal{L}_{\mathrm{MCR}}$ could measure the volume of the space between $Z$ and $Z^{\prime}$. The proof will be elaborated in the Section \ref{sec:5}. 

\subsection{Loss Functions}
\label{loss}
We incorporate the data manifold constraint as regularization terms in the loss functions of GAN models, as depicted in Eq.\ref{eq6} and Eq.\ref{eq7}, where ${\lambda}$ and $\gamma$ are hyperparameters. Specifically, we optimize the discriminator by maximizing $\mathcal{L}_{\mathrm{MCR}}$, which encourages it to focus on the underlying data manifold and push the manifold of the generated samples to be misaligned with that of the real images. Conversely, minimizing $\mathcal{L}_{\mathrm{MCR}}$ guides the generator towards producing samples that align closely with the real image manifold. This dual strategy sharpens the discriminator's ability to differentiate and guides the generator towards producing higher-quality, more diverse images. 
\begin{equation}
\label{eq6}
\begin{split}
\min _{\mathrm{D}} \mathcal{L}_{\mathrm{D}} = & \underset{z \sim p_{\mathrm{z}}}{\mathbb{E}}\left[D(G(z))\right]-\underset{x \sim {P_{\mathrm{data}}}}{\mathbb{E}}\left[D(x)\right] \\
& -{\lambda} \mathcal{L}_{\mathrm{MCR}}(Z,Z^{\prime})
\end{split}
\end{equation}
\begin{eqnarray}
\label{eq7}
\min _{\mathrm{G}} \mathcal{L}_{\mathrm{G}} & = & -\underset{z \sim p_{\mathrm{z}}}{\mathbb{E}}\left[D(G(z))\right]+ \gamma \mathcal{L}_{\mathrm{MCR}}(Z,Z^{\prime})
\end{eqnarray}
The representations $Z$ and $Z^{\prime}$ are learned by the discriminator. It has been shown that different network layers are responsible for different levels of detail in the images. Empirically, the latter blocks of the network have more effect on the style (e.g. texture and color) of the image whereas the earlier blocks impact the coarse structure or content of the image \cite{styleclip}. Thus, we choose features from a shallow network layer of the discriminator as representations $Z$ and $Z^{\prime}$ in our experiments.

While the effectiveness of $\mathcal{L}_{\mathrm{MCR}}$ relies on knowing the submanifold membership matrices $C$, obtaining this knowledge can be expensive or impractical in unsupervised settings. To address this challenge, we propose a novel unsupervised approach to learn this information directly from the data.

Specifically, we employ a three-layer Multi-Layer Perceptron (MLP) network, denoted as $\boldsymbol{F}$, to learn the relationship between data points in the feature space. This network takes the features $Z$ from the discriminator as input and outputs a matrix $M \in \mathbb{R}^{n \times n}$. The elements of the $j$-th row of the $M$ matrix are the diagonal elements of the $C^j$ matrix. Hence, $C^j(i,i)=M(j,i)$, where $i \in \left\{1,...,n\right\}$. The key idea is that elements along each row of $M$ correspond to the submanifold membership probabilities for a particular data point. In other words, a higher value $M(j, i)$ indicates a higher likelihood that data points $x_i$ and $x_j$ belong to the same submanifold.

In order to train the network $\boldsymbol{F}$, we employ a pre-trained encoder network, denoted as $f_{\mathrm{pre}}$, to extract features for each data point. These features are represented by the matrix $\bar Z = \left\{ {\bar z_1,...,\bar z_{n}} \right\}$, where $\bar z_i = {f_{\mathrm{pre}}}\left( {x_i} \right)$, $i \in \left\{ {1,...,n} \right\}$. We define a similarity measure based on the Euclidean distance between feature vectors. Regarding the $j$-th element ${{{\bar z}_j}}$ in $\bar Z$ as the anchor, we define that:
\begin{equation}
M^{pro}(j,i) = \exp ( -{{{\| {{{\bar z}_i} - {{\bar z}_j}} \|_2^2} \mathord{\left/
 {\vphantom {{\left\| {{{\bar z}_i} - {{\bar z}_j}} \right\|_2^2} \tau }} \right.
 \kern-\nulldelimiterspace} \tau }} )
\end{equation}
where $\tau$ is the temperature hyperparameter,  ${{\bar z}_i}$ is the $i$-th sample in $\bar Z$. Then, we can obtain:
\begin{equation}
M^{pro}_j = \left[ {M^{pro}(j,1),..., M^{pro}(j,n)} \right]
\end{equation}
We in turn treat the samples in $\bar Z$ as anchors and obtain ${M^{pro}} = {\left[ {M^{pro}_1,...,M^{pro}_n} \right]^{\mathrm{T}}}$. To this end, we give the prior constraint ${M^{pro}}$ of $M$ based on the similarity of different pairs of samples. The loss function of network $\boldsymbol{F}$ is as follows:
\begin{equation}\label{dswd}
{\mathcal{L}_{con}} = \left\| {{M^{pro}} - \boldsymbol{F}(Z)} \right\|_2^2
\end{equation}

The network $\boldsymbol{F}$ is first pretrained on the training dataset, with $Z$ being extracted by $f_{\mathrm{pre}}$. Then, it is trained together with the discriminator. In the optimal case, the network $\boldsymbol{F}$ will learn the relationship between the samples, and $M(i,j)$ will be close to 1 for data points belonging to the same submanifold, and close to 0 for points from different submanifolds. In our experiments, we find that the choice of the pre-trained encoder has little impact on the final results.




\section{Theoretical Analysis}
\label{sec:5}

In this section, we present theoretical analysis of the proposed method, MCR. Firstly, we establish a theoretical connection between the compactness measure $\mathcal{V}(Z)$ and information theory. Secondly, we prove that $\mathcal{L}_{\mathrm{Tr}}$ can evaluate the data manifold of features, and $\mathcal{L}_{\mathrm{MCR}}$ can measure the disparity of the manifolds between generated images and real images.

We begin by demonstrating the connection between our proposed measure, denoted by $\mathcal{V}(Z)$, and information theory. 

\begin{proposition}
\label{p1}
$\mathcal{V}(Z)=\frac{1}{2n} \operatorname { T r }\left(Z Z^{\mathrm{T}}\right)$ where $Z=\left[z^{1}, \ldots, z^{n}\right] \subset \mathbb{R}^{d \times n}$ can measure the compactness of a distribution from its finite samples $Z$.
\end{proposition}

\begin{proof}

Based on the first-order Taylor series approximation, $\log \operatorname{det}(\mathbf{C}+\mathbf{D}) \approx \log \operatorname{det}(\mathbf{C})+\operatorname{Tr}\left(\mathbf{D}^{\mathrm{T}} \mathbf{C}^{-1}\right)$, we can get following equations.
\begin{equation*}
   \begin{aligned}
    \frac{1}{2n}\operatorname { T r }(Z Z^{\mathrm{T}}) &= \frac{1}{2} ( \frac{1}{n} \operatorname { T r }(Z Z^{\mathrm{T}}) + \log \operatorname{det}(\boldsymbol{I}))\\
    &\approx  \frac{1}{2} \log \operatorname{det}(\boldsymbol{I}+ \frac{1}{n} Z Z^{\mathrm{T}})
   \end{aligned}
\end{equation*}

In information theory, rate distortion can be used to measure the “compactness” of a random distribution \cite{cover2006elements}. Given a random variable $z$ and a prescribed precision $\epsilon > 0$, the rate distortion $R(z,\epsilon )$ is the minimal number of binary bits needed to encode $z$ such that the expected decoding error is less than $\epsilon $. Given finite samples $Z=\left[z^{1}, \ldots, z^{m}\right] \subset \mathbb{R}^{d \times m}$, the average number of bits needed is given by the following expression: $\mathcal{L}(Z, \epsilon) \doteq\left(\frac{m+d}{2m}\right) \log \operatorname{det}\left(\boldsymbol{I}+\frac{d}{m \epsilon^{2}} Z Z^{\mathrm{T}}\right)$. As the sample size $m$ is large, our approach can be seen as an approximation of the rate distortion, which completes our proof. 
\end{proof}

Based on the above derivation, the proposed measure can be rewritten as: 
\begin{equation}
\begin{split}
\mathcal{L}_{\mathrm{Tr}} \approx & \frac{1}{2} \log \operatorname{det}(\boldsymbol{I}+\alpha Z Z^{\mathrm{T}}) \\
& -\sum_{j=1}^{k} \frac{\gamma_{j}}{2} \log \operatorname{det}(\boldsymbol{I}+\alpha_{j} Z C^{j} Z^{\mathrm{T}})
\end{split}
\end{equation}
where $\alpha=\frac{d}{n \epsilon^{2}}, \alpha_{j}=\frac{d}{\operatorname{Tr}\left(C^{j}\right) \epsilon^{2}}, \gamma_{j}=\frac{\operatorname{Tr}\left(C^{j}\right)}{n} \text { for } j=1, \ldots, k$. 

Based on this proposition, we can infer that the optimal solution of Eq.\ref{eq4} have following properties:
\begin{theorem}
\label{t1}
Suppose $\boldsymbol{Z}^{*}$ is the optimal solution that maximizes the function Eq.\ref{eq4}. We have:

- Between-submanifold Discrepancy: If the ambient space is adequately large, the subspaces are all orthogonal to each other, i.e.  $\left(\boldsymbol{Z}_{i}^{*}\right)^{\top} \boldsymbol{Z}_{j}^{*}=\mathbf{0}$  for $i \neq j$.

- Maximally Variance: If the coding precision is adequately high, i.e.,  $\epsilon^{4}<\min _{j}\left\{\frac{n_{j}}{n} \frac{d^{2}}{d_{j}^{2}}\right\}$, each subspace achieves its maximal dimension, i.e. $\operatorname{rank}\left(\boldsymbol{Z}_{j}^{*}\right)=d_{j}$ .
\end{theorem}

The proof for Theorem \ref{t1} is provided in the Appendix. In summary, we initially assume that $\left(\boldsymbol{Z}_{i}^{*}\right)^{\top} \boldsymbol{Z}_{j}^{*} \neq \mathbf{0}$. Based on the Singular Value Decomposition (SVD) and the condition $\sum_{j=1}^{k} d_{j} \leq d$, we infer that $\boldsymbol{Z}^{*}$ cannot be the optimal solution of Eq.\ref{eq4}. This contradiction validates the \textit{Between-submanifold Discrepancy}. Using a similar approach, we demonstrate that the outcome of Eq.\ref{eq4} is related to the singular values of $\boldsymbol{Z}_{j}^{*}$. Under the condition $\epsilon^{4}<\min _{j}\left\{\frac{n_{j}}{n} \frac{d^{2}}{d_{j}^{2}}\right\}$, $\operatorname{rank}\left(\boldsymbol{Z}_{j}^{*}\right)=d_{j}$, which confirms the \textit{Maximally Variance}.

Hence, $\mathcal{L}_{\mathrm{Tr}}$ effectively evaluates the data manifold of the representations. Since the features of each submanifold, $Z_j$ and $Z^{\prime}_j$, are similar to subspaces or Gaussians, their “distance” can be measured by the rate distortion. 
\begin{equation}
\label{eq-down}
\begin{split}
\mathcal{L}_{\mathrm{inf}}(Z, Z^{\prime}) =& \frac{1}{2 n k}\sum_{j=1}^{k}(\operatorname{Tr}(\widetilde{Z}_j \widetilde{Z}^{\mathrm{T}}_j)- \frac{1}{2}\operatorname{Tr}(Z_j Z^{\mathrm{T}}_j) \\ 
& -\frac{1}{2}\operatorname{Tr}(Z^{\prime}_j {Z^{\prime}_j}^{\mathrm{T}})),
\end{split}
\end{equation}
where $Z_j$ and $Z^{\prime}_j$ represent the feature representations of real and generated images belonging to submanifold $j$, respectively. $\widetilde{Z}_j = Z_j \cup Z^{\prime}_j$, $j \in \left\{ {1,...,k} \right\}$, and $k$ is the number of submanifolds within the data.

The metric $\mathcal{L}_{\mathrm{inf}}$ quantifies the distance between the submanifolds of real and generated images. A smaller $\mathcal{L}_{\mathrm{inf}}$ value indicates that the submanifolds of the real and generated data are more closely aligned. Given that submanifolds are inherently smaller than the entire manifold, we can say that $\sum_{j=1}^{k}(\operatorname{Tr}(\widetilde{Z}_j \widetilde{Z}^{\mathrm{T}}_j) \leq \operatorname{Tr}(\widetilde{Z} \widetilde{Z}^{\mathrm{T}})$. This leads to the conclusion that $\mathcal{L}_{\mathrm{inf}} \leq \mathcal{L}_{\mathrm{MCR}}$. Essentially, minimizing $\mathcal{L}_{\mathrm{MCR}}$ during training prompts the generator to produce images that closely align with the real data manifold across all classes. This inherently minimizes the disparity between submanifolds, as indicated by a lower $\mathcal{L}_{\mathrm{inf}}$ value. Conversely, the discriminator aims to maximize the separation between these manifolds by maximizing $\mathcal{L}_{\mathrm{MCR}}$, while concurrently drawing the features within the same submanifold closer together.

\section{Experiments}
\label{sec:exp}
In this section, we first introduce the datasets used in our experiments. Next, we provide the details of our experimental setup. We then present the results, including both quantitative metrics and qualitative examples, to demonstrate the effectiveness of our proposed method compared to existing approaches. Additionally, we showcase how our method benefits downstream tasks. Finally, we provide more ablation and analysis of different components of our method. 

\subsection{Datasets}

We use three RS datasets to evaluate our method: the UC Merced Land Use (UCLand) Dataset \cite{UCM}, NWPU-RESISC45 (NWPU) Dataset \cite{NWPU} and PatternNet (PN) Dataset \cite{patternnet}. Their information is shown in Table \ref{table-data}.

\begin{table}[h]
\centering
\caption{Details of the Dataset}
\renewcommand\arraystretch{1.5}
\begin{tabular}{cccc}
\toprule
Attribute        & UCLand                                                                     & NWPU         & PN           \\ \midrule
Images per class & 100                                                                        & 700          & 800          \\
Scene Classes    & 21                                                                         & 45           & 38           \\
Resolution (m)   & 0.3                                                                        & 0.2-30       & 0.062-4.693  \\
Image Size       & $256\times256$                                                                    & $256\times256$      & $256\times256$      \\
Source           & \begin{tabular}[c]{@{}c@{}}United States \\ Geological Survey\end{tabular} & Google Earth & Google Earth \\ \bottomrule
\end{tabular}
\label{table-data}
\end{table}

The UC Merced Land Use Dataset is one of the most widely used datasets in the field of remote sensing scene classification. It has 21 scene categories, each with 100 images. Each image has the size $256\times256$ and a spatial resolution of 0.3m. The images in the dataset come from more than 20 cities in the United States, including Las Vegas, Los Angeles, Miami, Santa Barbara, and Seattle.

The NWPU-RESISC45 Dataset has 31,500 images covering more than 100 countries and regions around the world. It has 45 categories with 700 images in each category. Each image is $256\times256$ pixels in size. The spatial resolution of this dataset is up to 0.2m and the lowest is 30m. The images are varied in lighting, shooting angle, imaging conditions, and so on. 

The PatternNet Dataset is a large-scale high-resolution remote sensing dataset. It has 38 categories with 800 images in each category. Each image is $256\times256$ pixels in size. The spatial resolution of this dataset varies from 0.06m to 4.7m per pixel. The images in PatternNet are collected from Google Earth imagery or via the Google Map API for some US cities.  


\begin{figure*}[htpb]
\centering
\subfloat[]{\includegraphics[width=.5\columnwidth]{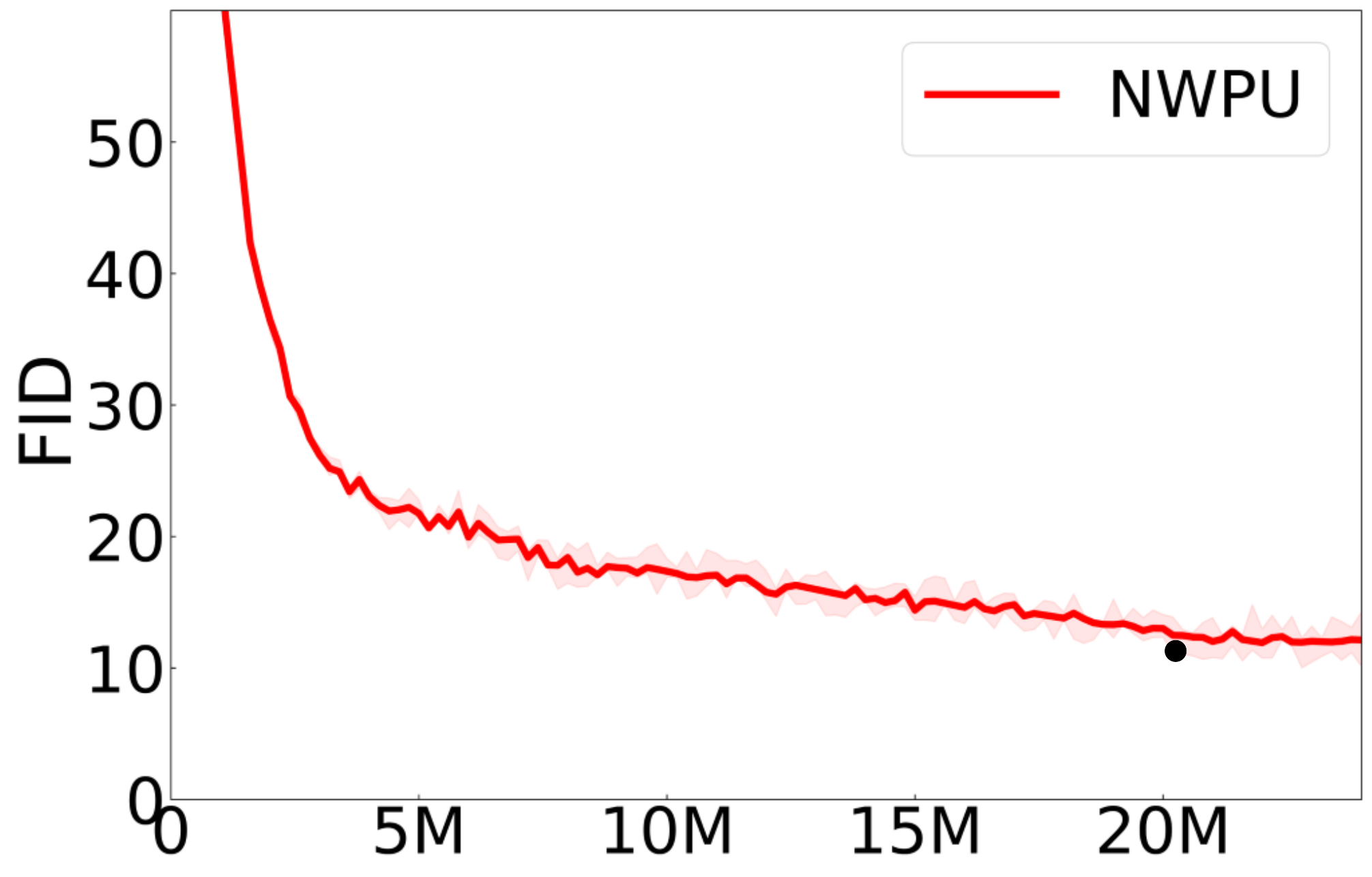}}
\subfloat[]{\includegraphics[width=.5\columnwidth]{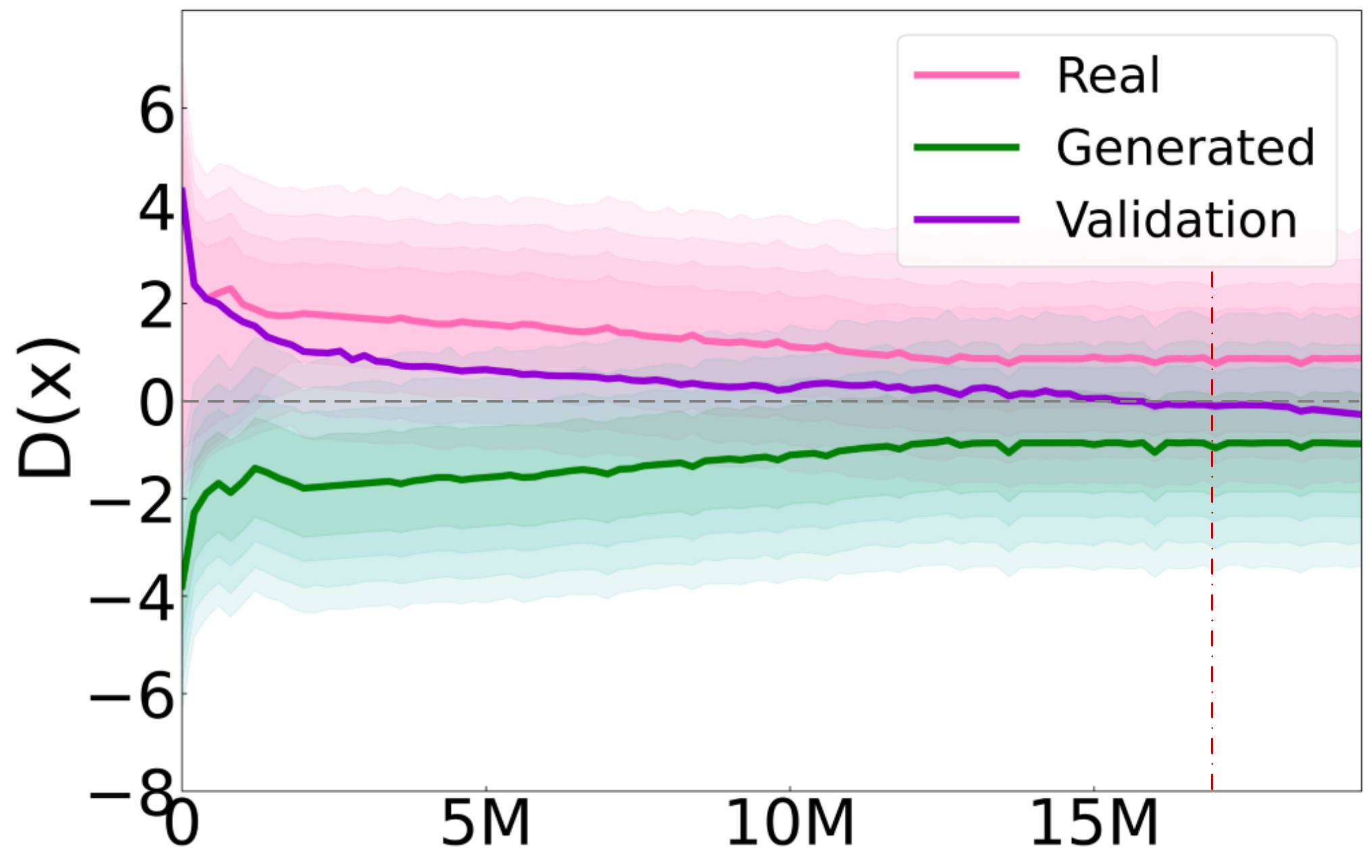}}
\subfloat[]{\includegraphics[width=.5\columnwidth]{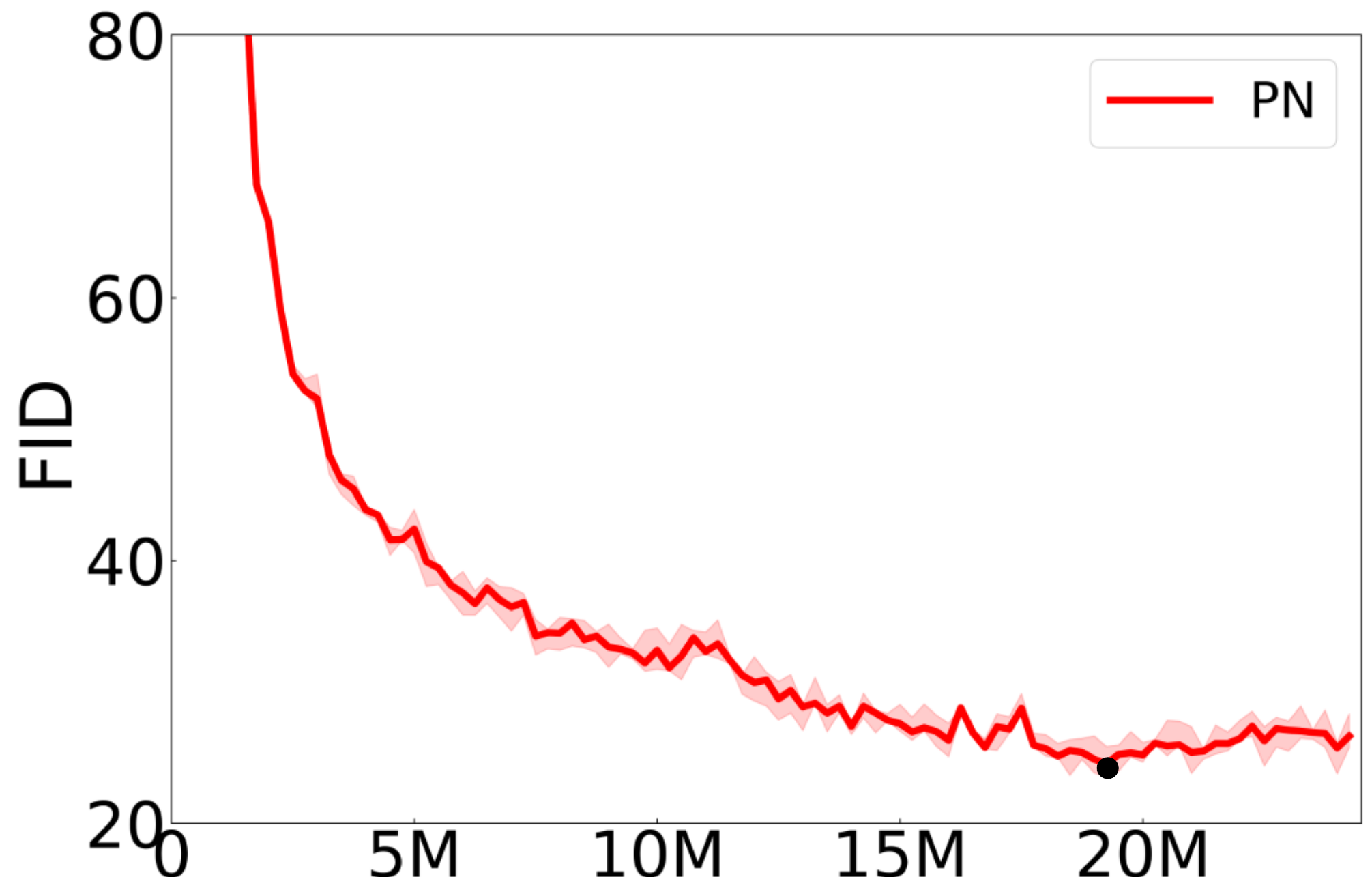}}
\subfloat[]{\includegraphics[width=.5\columnwidth]{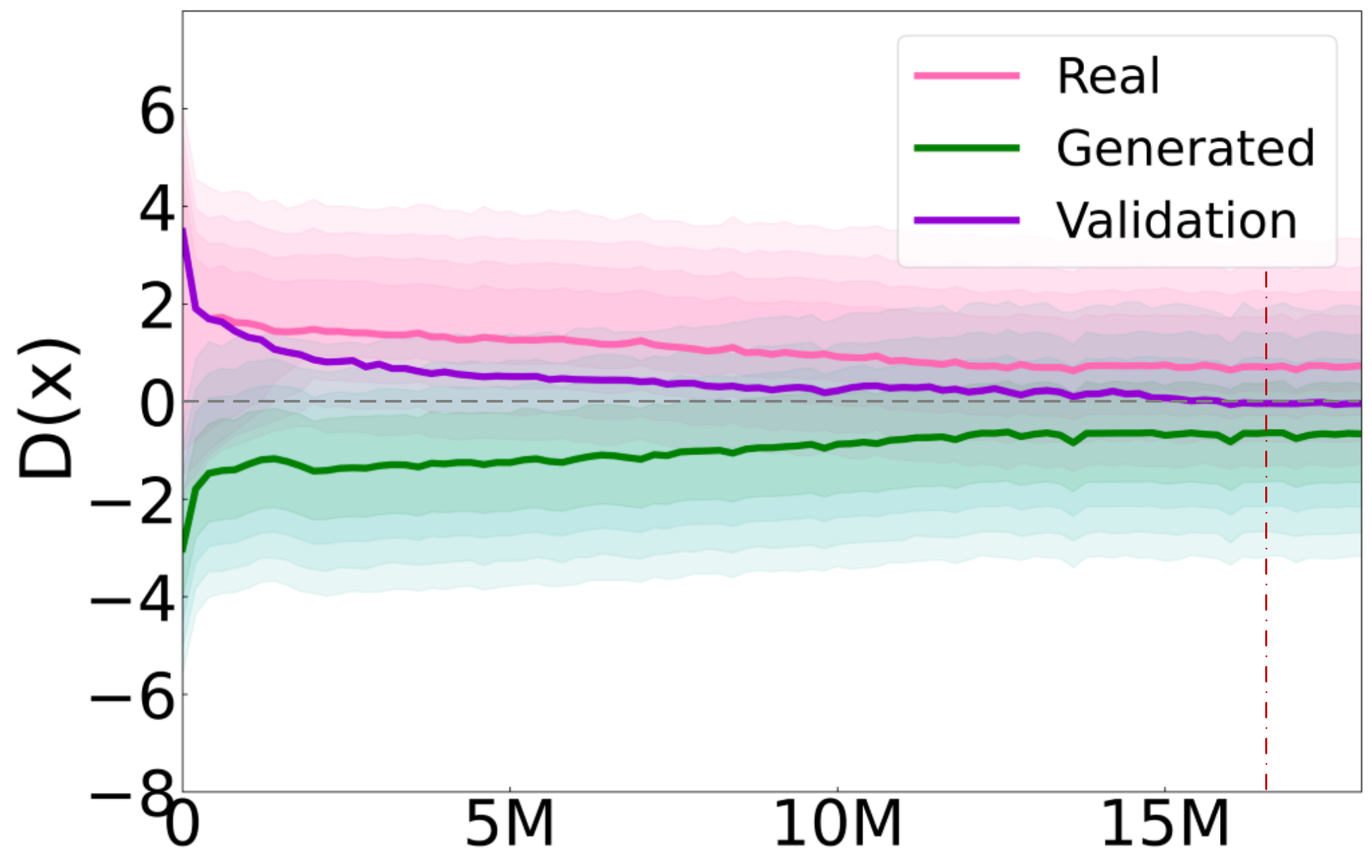}}
\caption{The horizontal axis indicates the training process (the number of real images shown to the discriminator). (a) Training curves of our method on NWPU dataset. (b) The outputs of the discriminator during training on NWPU dataset. (c) Training curves of our method on PN dataset. (d) The outputs of the discriminator during training on PN dataset. No divergence occurs during training, and the discriminator maintains high accuracy on the validation set. These findings suggest that MCR effectively alleviates the discriminator’s overfitting issue.}
\label{fig0-1}
\end{figure*}

\subsection{Experiment Setup}

\textbf{Baseline Methods}: We evaluate our approach to RS image generation by comparing it with established GAN models, including \textbf{BigGAN} \cite{biggan} and \textbf{StyleGAN2} \cite{StyleGAN2}. We also benchmark our method against techniques specifically developed to address GAN overfitting. These include the augmentation methods \textbf{ADA} \cite{stylegan2-ada} and \textbf{APA} \cite{APA}, regularization methods \textbf{LeCam} \cite{lecam}, \textbf{AdaptiveMix} \cite{liu2023adaptivemix} and \textbf{InsGen} \cite{yang2021data}, as well as few-shot generation methods \textbf{FastGAN} \cite{liu2020towards} and \textbf{MoCA} \cite{li2022prototype}.

\textbf{Implementation Details}: We use the official PyTorch implementation of StyleGAN2, ADA, FastGAN, InsGen, MoCA and AdaptiveMix. For BigGAN, APA and LeCam, we use the implementations provided by \cite{contragan}. Throughout our experiments, we set hyperparameters $\lambda=1$ and $\gamma=1$. Further details regarding the ablation study on these hyperparameters can be found in Section \ref{sec:ablation}.

\textbf{Evaluation Metrics}: We evaluate our method using Frechet inception distance (FID) \cite{FID}, as the most commonly-used metric for measuring the quality and diversity of images generated by GAN models. We also employ the Kernel Inception Distance (KID) \cite{KID}, a metric that remains unbiased by empirical bias \cite{xu2018empirical}, to further validate our results.

\subsection{RS Image Generation}
\begin{figure}[]
\centering
\setlength{\belowcaptionskip}{2pt}
\subfloat[]{\includegraphics[width=.7\columnwidth]{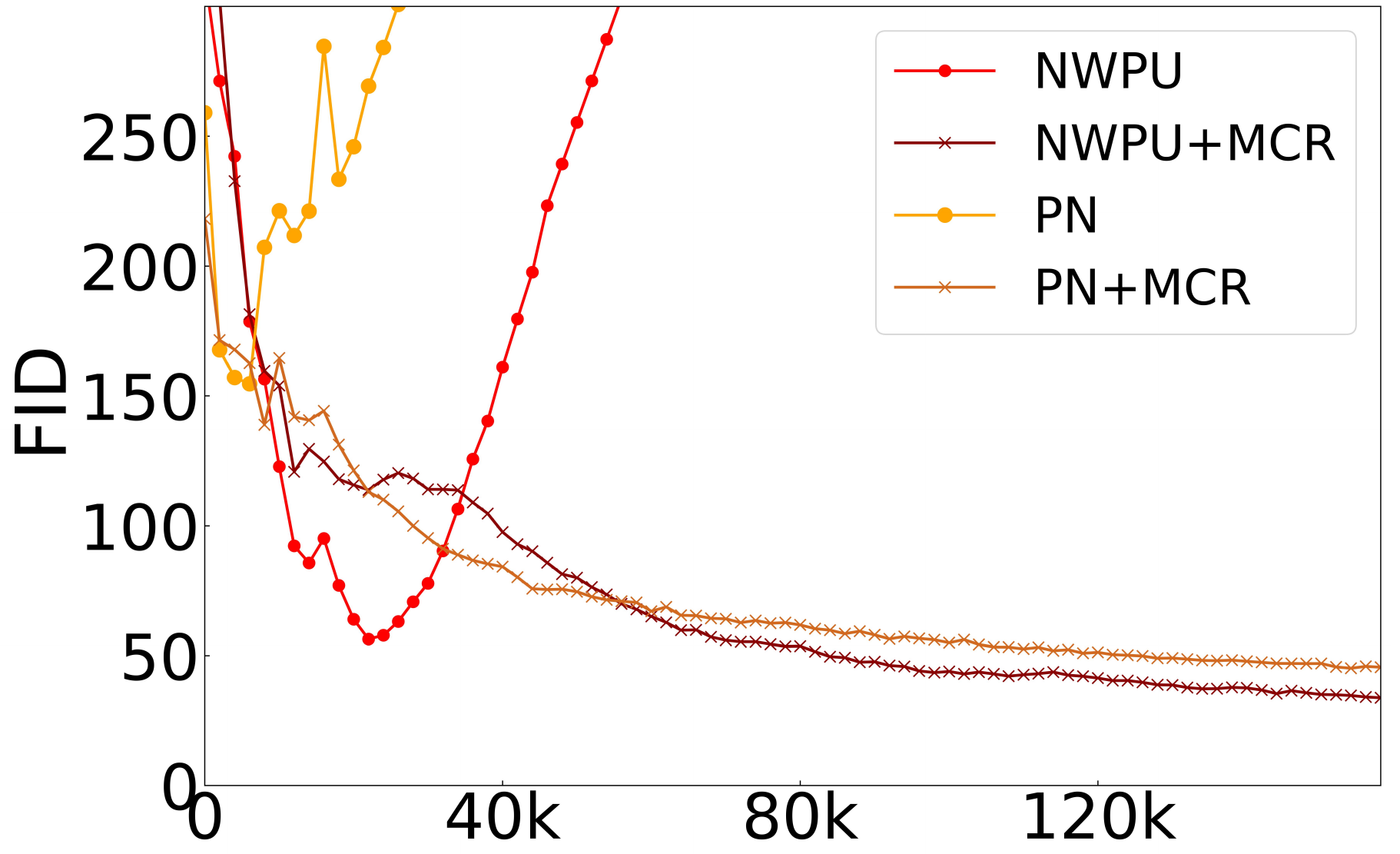}}\hspace{2pt}
\subfloat[]{\includegraphics[width=.7\columnwidth]{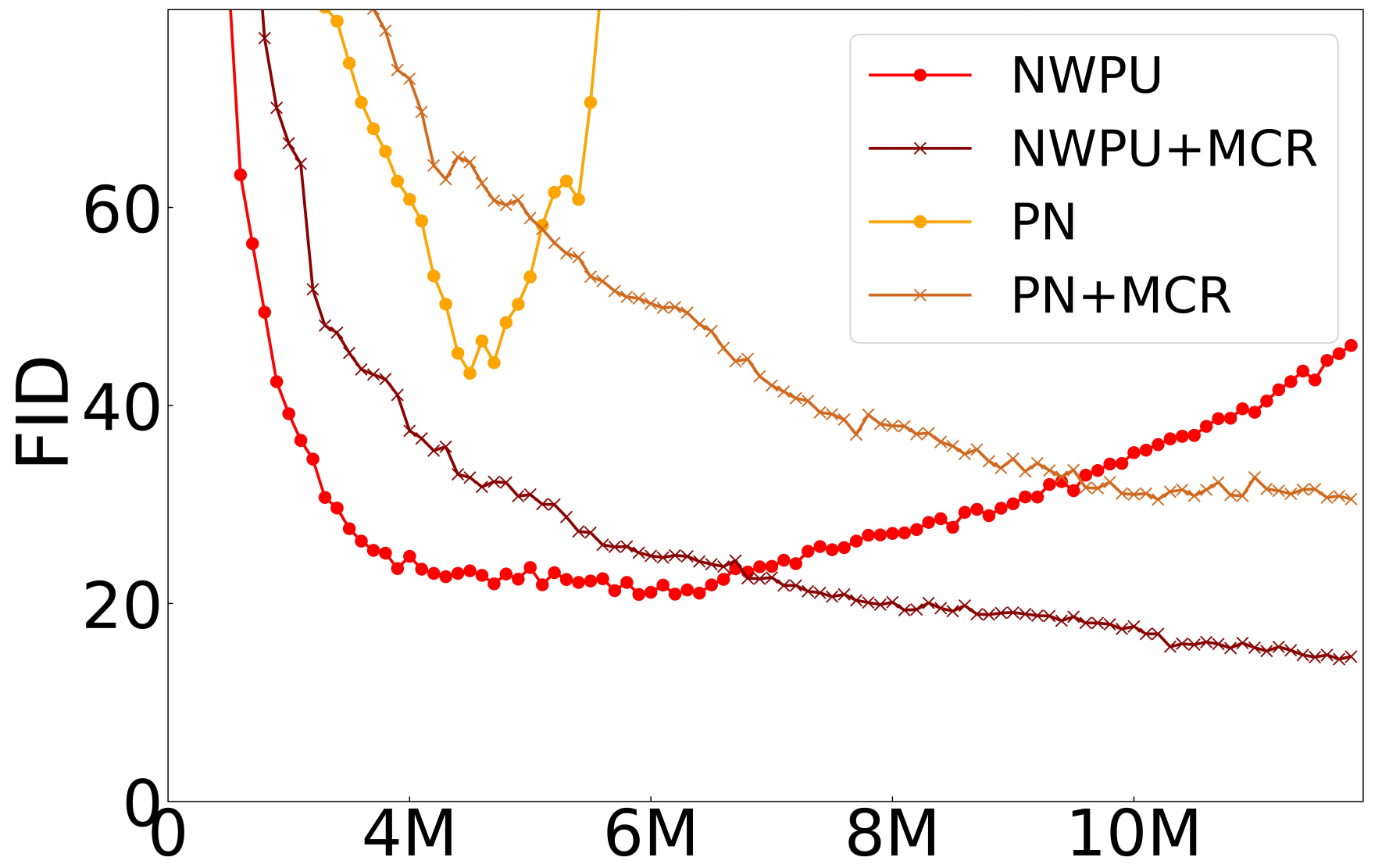}}\hspace{2pt}
\caption{(a) Training curves of BigGAN. (b) Training curves of StyleGAN2. Our method (MCR) not only reduces discriminator overfitting but also attains superior quality scores.}
\label{fig11}
\end{figure}

\begin{table}[htbp]
\centering
\caption{Comparison of Quality Scores (FID${\downarrow}$) on the UCLand, NWPU and PN Datasets}
\renewcommand\arraystretch{1.5}
\begin{tabular}{ccccc}
\toprule
\textbf{Methods} & \textbf{Backbone} & \textbf{UCLand} & \textbf{NWPU} & \textbf{PN} \\ \midrule
ADA              & StyleGAN2         & 74.25           & 11.97         & 33.53               \\
APA              & StyleGAN2         & 78.54           & 21.67         & 33.75               \\ \hline
AdaptiveMix      & StyleGAN2         & 70.95           & 13.53         & 34.63               \\
LeCam            & StyleGAN2         & 72.66           & 15.73         & 31.26               \\
InsGen           & StyleGAN2         & 95.65           & 10.92         & 50.76               \\ \hline
FastGAN          & --                & 76.63           & 32.17         & 51.39               \\
MoCA             & StyleGAN2         & 71.25           & 10.37         & 32.79               \\ \hline
MCR (Ours)             & StyleGAN2         & \textbf{69.43}  & \textbf{9.88} & \textbf{30.47}      \\ \bottomrule
\end{tabular}
\label{table1-1}
\end{table}

\textbf{Training Stability}. 
Fig.\ref{fig0-1} showcases the training process of our method MCR (based on StyleGAN2) on NWPU and PN datasets. Fig.\ref{fig0-1} (a) and (c) depict the FID curves during training, indicating no early divergence. Fig.\ref{fig0-1} (b) and (d) show the discriminator's accuracy on the validation set, consistently remaining high. These observations suggest that the proposed method MCR effectively mitigates the discriminator's overfitting problem.

\begin{table*}[ht!]
\centering
\caption{Comparison of Quality Scores (FID${\downarrow}$, KID($\times 10^{-3}$)${\downarrow}$) on the UCLand, NWPU and PN Datasets (the \textcolor{red}{red} numbers present our improvement)}
\renewcommand\arraystretch{1.5}
\begin{tabular}{ccccccc}
\toprule
\multirow{2}{*}{\textbf{Methods}} & \multicolumn{2}{c}{\textbf{UCLand}} & \multicolumn{2}{c}{\textbf{NWPU}} & \multicolumn{2}{c}{\textbf{PN}} \\ \cline{2-7} 
                                  & \textbf{FID}     & \textbf{KID}     & \textbf{FID}    & \textbf{KID}    & \textbf{FID}   & \textbf{KID}   \\ \midrule
\textbf{BigGAN+ADA}               & 98.09\textcolor{red}{-5.92}       & 52.31\textcolor{red}{-4.24}       & 30.91\textcolor{red}{-3.39}      & 11.13\textcolor{red}{-1.85}      & 63.94\textcolor{red}{-3.59}     & 36.71\textcolor{red}{-1.63}     \\
\textbf{StyleGAN2+ADA}            & 74.25\textcolor{red}{-4.22}       & 37.08\textcolor{red}{-4.08}       & 11.97\textcolor{red}{-1.88}      & 3.27\textcolor{red}{-0.92}       & 33.53\textcolor{red}{-2.96}     & 13.67\textcolor{red}{-2.32}     \\
\textbf{StyleGAN2+APA}            & 78.54\textcolor{red}{-4.39}       & 39.25\textcolor{red}{-3.79}       & 21.67\textcolor{red}{-2.78}      & 8.54\textcolor{red}{-2.03}       & 33.75\textcolor{red}{-4.84}     & 13.34\textcolor{red}{-2.99}     \\ \hline
\textbf{BigGAN+ADA+LeCam}         & 92.52\textcolor{red}{-5.27}       & 48.59\textcolor{red}{-3.52}       & 32.65\textcolor{red}{-2.57}      & 12.69\textcolor{red}{-1.87}      & 54.28\textcolor{red}{-4.42}     & 25.89\textcolor{red}{-2.37}     \\
\textbf{StyleGAN2+ADA+LeCam}      & 70.89\textcolor{red}{-3.91}       & 32.21\textcolor{red}{-2.86}       & 14.38\textcolor{red}{-2.34}      & 4.97\textcolor{red}{-1.47}       & 25.87\textcolor{red}{-2.86}     & 9.32\textcolor{red}{-1.73}      \\ \bottomrule
\end{tabular}
\label{table1}
\end{table*}

\textbf{Quantitative Comparison}.
Next, we present a quantitative comparison with established baselines. Fig.\ref{fig11} compares the training curves of our method on NWPU and PN datasets against BigGAN (Fig.\ref{fig11}(a)) and StyleGAN2 (Fig.\ref{fig11}(b)). Our method not only alleviates discriminator overfitting but also achieves superior quality scores. 

To further demonstrate the effectiveness of our proposed method MCR, we compare its performance against established methods for addressing discriminator overfitting. These methods typically fall into three categories: data augmentation, regularization, and architectural improvements. Although MCR falls under the regularization category, we compare it with representatives from all three categories. Table \ref{table1-1} summarizes the experimental results on the UCLand, NWPU, and PN datasets. Our method consistently outperforms all others: achieving a 2.14\% improvement in FID on UCLand compared to the well-established method. Similarly, on NWPU, our method achieves an FID of 9.88 compared to MoCA's 10.37, demonstrating MCR's clear advantage. Finally, on the PN dataset, our method surpasses LeCam by 2.53\%. It's important to note that these results are based on the same StyleGAN2 architecture. Overall, MCR outperforms the second-best method by approximately 3.13\% in terms of FID score. Table \ref{table1-1} clearly highlights the consistent superiority of our method in addressing the overfitting challenge.

\begin{figure*}[]
\centering
\includegraphics[width=0.9\textwidth]{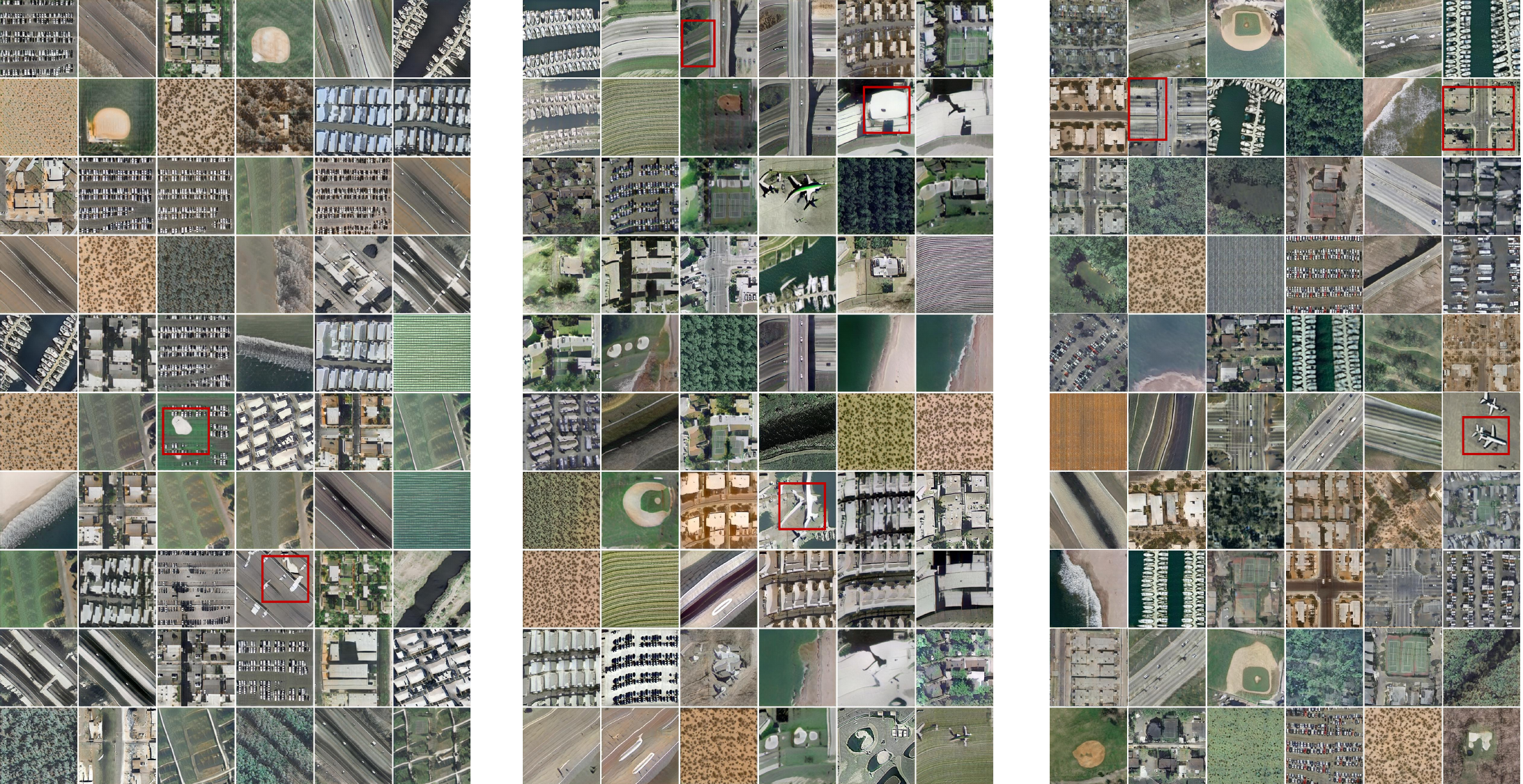}
\caption{Experiment results on the UCLand Dataset. The generated images of StyleGAN2 (\textbf{left}), StyleGAN2+AdaptiveMix (\textbf{middle}) and our method (\textbf{right}). The images generated by our method exhibit higher quality and diversity.}
\label{fig4}
\end{figure*}

\textbf{Versatility}. To verify the versatility of the method in this paper, we conduct experiments on different combinations of network architectures, enhancement methods and regularization methods. Experimental results are reported in Table \ref{table1}. The red numbers in Table \ref{table1} indicate the improvement of the GAN models after using our method. In general, the FID and KID scores for our proposed method indicate a significant and consistent advantage over all the compared methods. In detail, our comparison experiments can be divided into three types. First, under different model structures, BigGAN and StyleGAN2, our method is robust, which proves that our approach can be applied to any model architecture. Second, under different augmentation methods, ADA and APA, our method still performs well. This shows that the proposed approaches can be applied to other GAN models along with existing augmentation approaches. Third, combined with the regularization method LeCam, our method is still effective. The proposed method can be viewed as an effective complement to existing regularization methods.

\textbf{Qualitative Comparison}.
To validate the effectiveness of the methodologies presented in this paper, we conduct qualitative experiments on three RS datasets. For each dataset, we select a superior baseline method for benchmarking purposes. Specifically, on the UCLand dataset, our method is compared against StyleGAN2 and its extension, StyleGAN2+AdaptiveMix. On the NWPU dataset, our method is compared against StyleGAN2 and StyleGAN2+MoCA. Lastly, for the PN dataset, we evaluate our method against StyleGAN2 and StyleGAN2+LeCam. The generated images on the UCLand dataset are presented in in Fig. \ref{fig4}. 

We randomly select 60 samples from the generated images. The baseline StyleGAN2 model generates similar images, such as multiple parking lot images. In contrast, the imagery produced through our approach demonstrates a more uniform and varied distribution. Upon closer inspection, the images crafted using our method exhibit a higher degree of realism, with the shapes of objects appearing more regular and well-defined, which are marked in the red boxes. Further visual comparisons on the NWPU and PN datasets are presented in the Appendix.

\begin{table*}[htbp]
\centering
\caption{Comparison of Unsupervised Classification Accuracy(\%) on the UCLand Dataset and NWPU Dataset}
\renewcommand\arraystretch{1.5}
\begin{tabular}{ccccc}
\toprule
\textbf{\begin{tabular}[c]{@{}c@{}}Datasets\\ (training ratio)\end{tabular}} & \textbf{UCLand(80\%)} & \textbf{UCLand(50\%)} & \textbf{NWPU(80\%)} & \textbf{NWPU(20\%)} \\ \midrule
\textbf{MartaGAN \cite{MartaGAN}}                                                            & 94.86±0.80            & 85.51±0.69            & 75.43±0.28          & 75.03±0.28          \\
\textbf{AttentionGAN \cite{AttentionGAN}}                                                        & \textbf{97.69±0.6}    & 89.06±0.50            & —                   & 77.99±0.19          \\
\textbf{StyleGAN2}                                                           & 95.29±0.31            & 87.32±0.22            & 80.13±0.39          & 76.53±0.17          \\
\textbf{StyleGAN2+ADA}                                                       & 95.71±0.42            & 89.05±0.33            & 82.40±0.17          & 78.16±0.23          \\
\textbf{StyleGAN2+MCR (Ours)}                                                                & 97.33±0.39            & \textbf{92.72±0.22}   & \textbf{84.82±0.42} & \textbf{80.35±0.13} \\ \bottomrule
\end{tabular}
\label{table1-3}
\end{table*}

\begin{table*}[]
\centering
\caption{Comparison of Self-supervised Classification Accuracy(\%) on the UCLand Dataset and NWPU Dataset}
\renewcommand\arraystretch{1.5}
\begin{tabular}{cccccccc}
\toprule
\multirow{2}{*}{\textbf{Method}} & \multirow{2}{*}{\textbf{Backbone}} & \multicolumn{3}{c}{\textbf{UCLand}}                         & \multicolumn{3}{c}{\textbf{NWPU}}                           \\ \cline{3-8} 
                                 &                                    & \textbf{Original} & \textbf{StyleGAN2} & \textbf{MCR(Ours)} & \textbf{Original} & \textbf{StyleGAN2} & \textbf{MCR(Ours)} \\ \midrule
\multirow{2}{*}{\textbf{SimCLR}} & ResNet18                           & 67.59±0.38             & 69.61±0.92              & \textbf{70.48±0.88}     & 68.86±0.91             & 69.43±0.22              & \textbf{71.24±0.19}     \\
                                 & ResNet50                           & 69.39±0.67             & 70.25±0.28              & \textbf{71.80±0.35}     & 69.72±0.52             & 71.16±0.75              & \textbf{72.60±0.23}     \\ \hline
\multirow{2}{*}{\textbf{MAE}}    & ViT-base                           & 93.23±0.41             & 94.34±0.42              & \textbf{94.47±0.38}     & 89.84±0.28             & 90.21±0.19              & \textbf{90.29±0.53}     \\
                                 & Swin-T                             & 92.71±0.26             & \textbf{93.44±0.32}     & 93.33±0.17              & 90.21±0.43             & 89.87±0.59              & \textbf{91.03±0.22}     \\ \bottomrule
\end{tabular}
\label{table1-4}
\end{table*}

\subsection{Unsupervised Classification}

We employ the unsupervised classification task as a downstream experiment, aiming to both validate the effectiveness of MCR and to demonstrate its practical application in real-world tasks. We conduct two primary experiments. The first one operates at the feature level, utilizing the discriminator of the GAN as a feature extractor to perform unsupervised classification based on the extracted features. The second experiment operates at the image level, augmenting the original dataset with images generated by the GAN, and executing a self-supervised classification task based on this augmented dataset. We select the UCLand and NWPU datasets for testing the downstream task due to their inherent challenges for unsupervised classification; one dataset has the least amount of data while the other has the most categories.

In the first experiment, we use the learned representations $Z$ of the discriminator as features and apply a regularized linear L2-SVM classifier, adhering to the methodology used in prior studies \cite{MartaGAN}\cite{AttentionGAN}. The results, presented in Table \ref{table1-3}, underscore the superior performance of our method compared to StyleGAN2 and StyleGAN2+ADA across all tests. Our approach consistently outperforms in the majority of the experiments, affirming our method's ability to learn enhanced representations that facilitate precise data classification. Specifically, our method surpasses the baseline method by 4.16\% on the UCLand dataset. Similarly, on the NWPU dataset, our method outperforms the baseline method by 4.80\% and the second-best method by 2.62\%. These experiments further validate that by applying the manifold constraint, the discriminator can concentrate on the underlying data manifold and capture its essential characteristics.

In the second experiment, we utilize self-supervised algorithms SimCLR \cite{simclr} and MAE \cite{mae} as classification methods, choosing two distinct network architectures for each algorithm. All these models are initially pre-trained on the ImageNet dataset. We then augment the UCLand dataset with 1,050 generated images and the NWPU dataset with 9,000 generated images. These images are generated using traditional data augmentation, StyleGAN2, and our proposed method, respectively. Following this, we train the self-supervised models on these augmented RS datasets. These trained models are then used as feature extractors, and a regularized linear L2-SVM is employed for classification. As shown in Tables \ref{table1-4}, the algorithms trained on the dataset enhanced by our proposed methods consistently outperform those trained on datasets augmented by other approaches in terms of accuracy in most experiments. With an average classification accuracy improvement of 2.32\%, these results underscore the effectiveness of our proposed method in downstream tasks.

\subsection{Ablation Study}
\label{sec:ablation}
In this section, we provide further ablation and analysis over different components of our method.

\begin{table}[htbp]
\centering
\caption{Ablation Study on Relationship Matrix $C$}
\renewcommand\arraystretch{1.5}
\begin{tabular}{llll}
\hline
\toprule
\multicolumn{1}{c}{Methods} & \multicolumn{1}{c}{UCLand} & \multicolumn{1}{c}{NWPU} & PN \\ 
\midrule
Baseline                    & 74.25                      & 11.97                             & 33.53      \\
K-means (20)                    & 71.22                      & 11.31                             & 32.78      \\
K-means (40)                    & 72.28                      & 11.02                            & 31.33      \\ \hline
learnable (SimCLR)                    & \textbf{70.03}                      & 10.09                             & \textbf{30.57}      \\
learnable (CLIP)                    & 70.98                      & \textbf{9.89}                             & 30.71      \\
learnable (ResNet)                   & 70.81                      & 10.82                             & 30.78      \\
\bottomrule
\end{tabular}
\label{table4}
\end{table}

\begin{table}[htbp]
\centering
\caption{Ablation Study on Features from Different Blocks}
\renewcommand\arraystretch{1.2}
\begin{tabular}{lccccc}
\hline
\toprule
\textbf{Block} & 4    & 6    & 8    & 10  
  & 12    \\
\midrule
\textbf{FID}  & 14.92  & 14.05 &  \textbf{10.09} &  16.23  &  23.19  \\
\bottomrule
\end{tabular}
\label{table6-1}
\end{table}

\begin{table}[htbp]
\centering
\caption{Ablation Study on Regularizing Generator vs. Discriminator.}
\begin{tabular}{llll}
\hline
\toprule
\multicolumn{1}{c}{Methods} & \multicolumn{1}{c}{UCLand} & \multicolumn{1}{c}{NWPU} & PN \\ 
\midrule
Baseline                    & 74.25                      & 11.97                             & 33.53      \\
Only G                      & 73.95                      & 11.85                             & 33.06      \\
Only D                      & 71.23                      & 10.44                             & 31.83      \\
Ours                        & \textbf{70.03}                      & \textbf{10.09}                             & \textbf{30.57}      \\ 
\bottomrule
\end{tabular}
\label{table3}
\end{table}

\begin{table}[htbp]
\centering
\caption{Ablation Study on the Hyperparameter $\lambda$}
\begin{tabular}{llll}
\hline
\toprule
\multicolumn{1}{c}{{$\lambda$}} & \multicolumn{1}{c}{UCLand} & \multicolumn{1}{c}{NWPU} & PN \\ 
\midrule
0.1                                        & 80.60                      & 16.93                             & 40.09      \\
0.5                                        & 74.43                      & 12.91                             & 34.92      \\
0.7                                        & \textbf{69.26}                      & 10.68                             & 30.81      \\
1                                          & 69.43                      & \textbf{9.88}                             & \textbf{30.47}      \\ 
3                                        & 73.22                      & 12.31                             & 33.78      \\
\bottomrule
\end{tabular}
\label{table5}
\end{table}

\begin{table}[htbp]
\centering
\caption{Ablation Study on the Hyperparameter $\gamma$}
\begin{tabular}{llll}
\hline
\toprule
\multicolumn{1}{c}{{$\gamma$}} & \multicolumn{1}{c}{UCLand} & \multicolumn{1}{c}{NWPU} & PN \\ 
\midrule
0.1                                        & 78.29                      & 18.87                             & 41.13      \\
0.5                                        & 74.62                      & 13.16                             & 34.92      \\
0.7                                        & 71.26                      & 10.73                             & 31.92      \\
1                                          &  \textbf{69.43}                      & \textbf{9.88}                             & \textbf{30.47}      \\ 
3                                        & 73.24                      & 13.09                             & 35.07      \\
\bottomrule
\end{tabular}
\label{table6-2}
\end{table}

\textbf{Relationship Matrix $C$}. We employ various methods to construct the relationship matrix $C$. The first approach leverages SimCLR and K-means. We utilize a pre-trained SimCLR model to extract RS data features, followed by clustering using K-means. The centroids resulting from this clustering process serve as prototypes of the training data. In this experiment, we consider 20 and 40 clustering centroids, respectively. The second approach is a learnable matrix introduced in this paper. We evaluate it using pre-trained encoders such as SimCLR, CLIP \cite{clip}, and a ResNet50 model (pretrained on ImageNet). We utilize StyleGAN2+ADA as the baseline method. Table \ref{table4} presents the results of our ablation study on the UCLand, NWPU, and PN datasets. Notably, the choice of the pre-trained encoder has minimal impact on the results. In subsequent experiments, we opt for the learnable approach and the SimCLR model as the pre-trained encoder.

\textbf{Representations $Z$}. We conduct ablation studies on features from different network layers. As different network layers are related to different levels of details in the generated image, and the earlier blocks of the network impact the coarse structure or content of the image. We conduct experiments on NWPU dataset with StyleGAN2+ADA model, which consists of 14 blocks. We choose 5 blocks for comparison, and the results are shown in Table \ref{table6-1}. We empirically choose the outputs of 8th block as the representations $Z$.

\textbf{Regularizing generator vs. discriminator}. Our default method add regularization on the loss functions of both generator $G$ and discriminator $D$. In this experiment, we investigate the effectiveness of separately regularizing $G$ and $D$. We utilize StyleGAN2+ADA as the baseline method. Table \ref{table3} presents the results of the ablation study on UCLand, NWPU and PN datasets. The No Regularization version yields poor results as expected. Adding the regularization method on $D$ already brings significant improvement to the model under different datasets. As proposed in our final method, adding the regularization on both $G$ and $D$ achieves the best results.

\textbf{Hyperparameters}. We conduct the ablation study on the hyperparameters ${\lambda}$ and ${\gamma}$ using StyleGAN2+MCR on UCLand, NWPU and PN datasets. The FID scores are shown in Table \ref{table5} and Table \ref{table6-2}. Based on the experiment results, we set ${\lambda}=1$ and ${\gamma}=1$ in the following experiments.

\section{Conclusion}
\label{sec:con}

In this study, we aim to address the challenges posed by RS images in the context of GANs. We observe that RS images exhibit higher intrinsic dimensions compared to natural images, resulting in difficulties for the discriminator and an increased risk of overfitting. To mitigate these issues, we introduce a novel measure to capture the real data manifold and propose the MCR method to effectively address the discriminator overfitting while enhancing the generator’s performance. We also present innovative learning paradigms for the unsupervised generation of RS images. Our method’s effectiveness is confirmed through theoretical analysis and comprehensive experiments on three RS datasets using different GAN models. This demonstrates the adaptability and efficiency of our approach.

As for future works, there are several intriguing avenues for further research. Further exploration into the application of the MCR method to other types of generative models could yield interesting insights. Another valuable direction would be investigating the impact of different RS image characteristics, such as varying resolutions or spectral ranges, on GAN performance. This could provide crucial information to refine our approach further. By pursuing these lines of inquiry, we aim to continue enhancing the capabilities of GANs in handling the intricacies of RS images.

\bibliographystyle{IEEEtran}
\bibliography{library_abbreviated}


\end{document}